\documentclass{article}
\usepackage{ijcai18}

\usepackage{times}
\usepackage{xcolor}
\usepackage{soul}
\usepackage[utf8]{inputenc}
\usepackage[small]{caption}
\usepackage{amsmath,amssymb,amstext,amsthm,bbm} 
\usepackage[pdftex]{graphicx} 
\usepackage{mathtools}
\usepackage{subcaption}
\usepackage{caption}
\usepackage{algorithm}
\usepackage{algorithmic}
\usepackage{array, multirow}
\usepackage{ctable}
\usepackage{framed}
\usepackage{standalone}
\usepackage{url}

\DeclarePairedDelimiter\abs{\lvert}{\rvert}%
\DeclarePairedDelimiter\norm{\lVert}{\rVert}%
\newtheorem{theorem}{Theorem}[section]

\newtheorem{lemma}[theorem]{Lemma}
\newtheorem{definition}[theorem]{Definition}

\newtheorem{assumption}[theorem]{Assumption}

\DeclareMathOperator*{\argmin}{arg\,min}

\DeclareMathOperator*{\sgn}{\mathbf{sgn}}

\DeclareMathOperator*{\tr}{\text{tr}}

\newif\iffull
\fulltrue 

\newif\ifshowauthor
\showauthortrue

\usepackage{standalone}
\title{Nonsmooth Frank-Wolfe using Uniform Affine Approximations}

\ifshowauthor
\author{Edward Cheung \phantom{{}\and{}} Yuying Li\\ 
Cheriton School of Computer Science, University of Waterloo, Waterloo, Canada  \\
\{eycheung, yuying\}@uwaterloo.ca
}
\fi

\begin{document}

\maketitle

\begin{abstract}

Frank-Wolfe methods (FW) have gained significant interest in the machine learning community due to its ability to efficiently solve large problems that admit a sparse structure (e.g. sparse vectors and low-rank matrices). However the performance of the existing FW method hinges on the quality of the linear approximation. This typically restricts FW to smooth functions for which the approximation quality,  indicated by a  global curvature measure, is reasonably good. 

In this paper, we propose a modified FW algorithm amenable to nonsmooth functions by optimizing for approximation quality over all affine functions given a neighborhood of interest. We analyze theoretical properties of the proposed algorithm and demonstrate that it overcomes many issues associated with existing methods in the context of nonsmooth low-rank matrix estimation.
\end{abstract}

\section{Introduction}
We are interested in solving problems of the form,
\begin{equation*}
\min_{X \in \mathcal{D}} f(X) ~ s.t. ~ \norm{X}_{\tr} \leq \delta
\end{equation*}
where the trace norm $\norm{X}_{\tr} = \sum \sigma_{i}(X)$, is the sum of the singular values of $X$. This problem is well studied in the case where $f$ is a smooth convex function. For example, in matrix completion,  many efficient algorithms have been proposed, including Frank-Wolfe \cite{thesisMJ}, active set methods \cite{hsieh2014nuclear}, and proximal methods \cite{parikh2014proximal}. 

Recently, there has also been interest in solving the trace norm constrained problem where the objective function is not differentiable, e.g.,
\begin{equation} \label{RMC}
\min_{X: \norm{X}_{\tr} \leq \delta} f(X) \equiv L(X) + \lambda_{1}\norm{X}_{1}.
\end{equation}
where $L(X)$ is an empirical loss function. Problem \eqref{RMC} has been found useful  \cite{richard2012estimation} for sparse covariance estimation and graph link prediction, for which solutions are expected to exhibit simultaneously sparse and low-rank structure. For problem \eqref{RMC}, proximal methods will likely fail to scale due to the full SVD required at each iteration.  In addition, the active set method in \cite{hsieh2014nuclear} utilizes second order information, making it unclear how to develop a scalable algorithm when the function is not differentiable, let alone twice differentiable.

FW appears to address these issues, only requiring first order information and one singular vector pair calculation per iteration. However, the linear minimization problem is no longer straightforward if the $\ell_{1}$ norm is added to the constraint set. Thus, we propose a variant of the FW algorithm to address nonsmooth objectives, focusing especially on low-rank matrix estimation problems. Nondifferentiability in the objective function often leads to an unbounded curvature constant, and standard convergence analysis can no longer be applied. Moreover, it becomes unclear how to  define the linear approximation appropriately since choosing an arbitrary subgradient often leads to inadequate local approximations, leading to poor empirical results.

To address these issues, we replace the traditional linear minimization problem by a Chebyshev uniform affine approximation. This modification allows for a well-defined linear optimization problem even when the objective is nonsmooth or has unbounded curvature. We demonstrate experimentally that this carefully selected linear minimization leads to significant improvement over a variety of matrix estimation problems, such as sparse covariance estimation, graph link prediction, and $\ell_{1}$-loss matrix completion.

\section{Background}
\subsection{Frank-Wolfe for Nonsmooth Functions}
The FW algorithm is a first-order method for solving
$\min_{x \in \mathcal D} f(x)$, where
$f(x)$ is a convex function and $\mathcal D$ is a convex and compact set \cite{frank1956algorithm}. 
The algorithm is motivated by replacing the objective function $f(x)$ with its first-order Taylor expansion and solving the first-order surrogate on the domain $\mathcal{D}$. Formally, the Frank-Wolfe algorithm solves the following linear minimization problem at the iteration $k$,
\begin{equation*}
s^{(k)} \coloneqq \argmin_{s \in \mathcal{D}} f(x^{(k)}) + \langle s - x^{(k)}, f(x^{(k)}) \rangle.
\end{equation*}
The next iterate is then given by a convex combination of $s^{(k)}$ and $x^{(k)}$, which guarantees that the resulting iterate remains feasible, assuming the initial $x_{0} \in \mathcal{D}$.

For smooth convex functions, the Frank-Wolfe algorithm is known to converge at a rate of $O(1/k)$. The convergence analysis relies on the  concept of \emph{curvature constant} \cite{clarkson2010coresets,thesisMJ}, which measures the quality of the linear approximation.

Let $f$ be a convex and differentiable function $f: \mathbb{R}^{n} \rightarrow \mathbb{R}$, and let $\mathcal{D}$ be a convex and compact subset of $\mathbb{R}^{n}$. Then, the curvature constant $C_f$ is defined as
\begin{equation*}
\label{def:curvature_constant}
C_{f} \coloneqq \sup_{\substack{x,s \in \mathcal{D}\\ \alpha \in [0,1] \\ y = x + \alpha(s - x)}} \frac{1}{\alpha^{2}}( f(y) - f(x) - \langle y - x, \nabla f(x) \rangle ).
\end{equation*}

 When the value of $C_{f}$ is large, it suggests that there are regions in $\mathcal{D}$ where the local linear approximation is poor. It can be seen in \cite{clarkson2010coresets,thesisMJ} that the curvature constant is closely related to the Lipschitz constant of the gradient. However, if $f$ is not differentiable, even for simple functions such as $f(x) = \lambda\norm{x}_{1}$, it is easy to verify that $C_{f}$ is unbounded regardless of the choice of subgradient. The convergence properties for nonsmooth functions become difficult to analyze, and in practice we see that using an arbitrary subgradient for FW typically performs poorly.

\subsection{Existing Work}
To motivate better local approximations, \cite{white1993extension} employs an approximate subdifferential, $T(x, \epsilon)$, which considers all the subgradients for any $y$ in an $\epsilon$-neighborhood of $x$. This idea attempts to find an appropriate linearization of the objective that is good within a specified neighborhood of interest. However, arbitrary approximate subgradients can still perform poorly in practice if they are not chosen carefully. In \cite{ravi2017deterministic}, the idea is extended and the authors proposed a generalized curvature constant to best determine which approximate subgradient to choose. However, this leads to an optimization problem which is not generally easy to solve and  it is not obvious how to extend these ideas to \eqref{RMC} efficiently.

In \cite{pierucci2014} and \cite{argyriou2014hybrid}, the objective in \eqref{RMC} is replaced by a smoothed objective. The gradients of the smoothed objective are given by,
\begin{equation*}
[G^{(k)}(X^{(k)})]_{ij} = \nabla L(X^{(k)}) + \begin{cases}
\lambda\sgn(X_{ij}), &\text{if } \abs{X^{(k)}_{ij}} \geq \mu\\
\frac{\lambda}{\mu}X^{(k)}_{ij}, &\text{if }  \abs{X^{(k)}_{ij}} < \mu.
\end{cases}
\end{equation*}
for the Smoothed Composite Conditional Gradient (SCCG) algorithm in \cite{pierucci2014}, and 
\begin{equation*}
G^{(k)}(X^{(k)}) = \nabla L(X^{(k)}) - \frac{1}{\beta^{(k)}}X_{k} + \frac{1}{\beta^{(k)}}S(X^{(k)}, \lambda\beta^{(k)})
\end{equation*}
for the Hybrid Conditional Gradient with Smoothing (HCGS) algorithm in \cite{argyriou2014hybrid}, where
\begin{equation*}
S(X^{(k)}, \lambda\beta^{(k)}) = \sgn(X) \odot \max\{\abs{X} -  \lambda\beta^{(k)}, 0\}
\end{equation*}
is the soft-thresholding operator. 
In both cases, the objective is smoothed by parameters $\mu$ and $\beta^{(k)}$ respectively, where the smooth approximation given is the best approximation across all $(1/\mu)$-smooth (or resp. $1/\beta^{(k)}$-smooth) functions. However, determining how smooth the approximation should be is not easy to know a priori and it often varies depending on the iterate. Empirically, we also observe there is a nontrival dependency between the smoothing parameters and convergence rate. At any given iteraiton, if the smoothing parameters are not set appropriately, the algorithm often makes no progress for many iterations.

In \cite{yao2016greedy}, a nonsmooth generalization to rank-one matrix pursuit is proposed which utilizes subgradients in the linearized subproblem. To ensure convergence, the rank-one update at each iteration is replaced by a rank-$k$ variant where $k$ is computed by taking however many leading singular vectors are required to ensure the solution to the rank-$k$ subproblem is not too far from the subgradient in norm. Since the subgradients are typically not low rank, the number of singular vectors required can often be very large (possibly requiring a full SVD), and this approach can still fail to scale in similar ways to proximal methods.

The work in \cite{odor2016frank} also considers Frank-Wolfe methods when the curvature constant is unbounded. However, the algorithm is is specific to the phase retrieval problem for which the objective is still differentiable, simplifying the analysis.

Lastly, the Generalized Forwards-Backwards (GenFB) algorithm was introduced to solve \eqref{RMC} by alternating between proximal steps using the trace and $\ell_{1}$ norm \cite{richard2012estimation}. As alluded to earlier, these algorithms tend to scale poorly due to the full SVD required at each iteration.

\subsection{Achieving a better linear approximation}
 The previous work on FW for nonsmooth minimization \eqref{RMC} shares a common idea, i.e., finding a meaningful way to define an appropriate linear optimization subproblem in a scalable manner. We consider directly minimizing the approximation error over all possible affine functions over a neighborhood specified carefully for the Frank-Wolfe steps. We will show that under modest assumptions, the linear subproblems we propose will be simple to solve and do not rely on specifying the desired level of smoothness, as required by methods discussed in the previous section.

\begin{definition}\label{def:ue}
Given some $r > 0$, the \textbf{uniform affine approximation} to a function $f: \mathbb{R}^{m \times n} \rightarrow \mathbb{R}$ is defined as $\ell(Y) = b + \langle Y - X, \xi\rangle$ where,
\begin{equation*}
(\xi, b) \in \argmin_{(\xi, b)}\max_{Y \in \bar{\mathcal{B}}_{\infty}(X, r)} \abs*{f(Y) - b - \langle Y - X, \xi \rangle}
\end{equation*}
and $\bar{\mathcal{B}}_{\infty}(X, r)$ is the closed element-wise infinity norm ball of radius $r$ around $X$.
\end{definition}

This motivates a natural variant of FW where, at each iteration, the linear subproblem using a subgradient is replaced with the uniform affine approximation. In particular, we can view the FW iterates as,
\begin{equation*}
X^{(k+1)} = X^{(k)} - \alpha^{(k)}(X^{(k)} - S)
\end{equation*}
for some $S \in \mathcal{D}$. Thus, $X^{(k+1)} \in \bar{\mathcal{B}}_{\infty}(X^{(k)}, \alpha^{(k)}\mathbf{diam}(\mathcal{D}))$ where $\mathbf{diam}(\mathcal{D}) = \max_{S \in \mathcal{D}}\norm{X^{(k)} - S}_{\infty}$. For FW, this implies that we can restrict our attention to a neighborhood around the current iterate $X^{(k)}$ where the neighborhood has radius $\alpha^{(k)}\mathbf{diam}(\mathcal{D})$.

Specifically, we observe that in FW there exist step size schedules, e.g., $\alpha^{(k)} = 2/(k+2)$, which are independent of the current iterate and guarantee convergence. Our proposed approach is to assume that such a step size schedule is specified \emph{a priori} and to use the uniform affine approximation for the FW subproblems. This allows the linear optimization subproblems to be defined a meaningful way that is related to the FW steps. Moreover, the subproblems no longer require $f$ to have bounded curvature or even to be differentiable.

\section{Frank-Wolfe with Uniform Approximations}
For the proposed uniform approximation approach to be viable, it is important that the uniform affine approximation can be calculated efficiently. We begin by considering real-valued functions and Chebyshev approximations.

\subsection{Chebyshev Approximations}
Given a real-valued function $f$ and an interval $[a,b] \subseteq \mathbf{dom}(f)$, the Chebyshev polynomial, denoted as $p_{d}(x)$, is a polynomial of degree not exceeding $d$ that best approximates $f$ on the interval $[a,b]$ in the uniform sense, 
\begin{equation*}
p_{d}(x) = \argmin_{p \in \Pi_{d}}\max_{a \leq x \leq b} \abs{f(x) - p(x)}
\end{equation*}
where $\Pi_{d}$ is the set of polynomials of degree at most $d$. 

\begin{theorem}[Chebyshev Equioscillation Theorem]\label{thm:cheb}
Let $f$ be a continuous function from $[a,b] \rightarrow \mathbb{R}$ and let $\Pi_{d}$ be the set of polynomials of degree less than or equal to $d$. Then 
\begin{equation*}
g^* = \argmin_{g \in \Pi_{d}}\norm{f - g}_{\infty}
\end{equation*}
if and only if there exists a $c \in \{-1, 1\}$ and $d + 2$ points $\{x_{1},...,x_{d+2}\}$ such that $a \leq x_{1} < ..., < x_{d+2} \leq b$ such that,
\begin{equation*}
f(x_{i}) - g^*(x_{i}) = c(-1)^{i}\norm{f - g^*}_{\infty}.
\end{equation*}
\end{theorem}
Although the equioscillation theorem only applies to a function of one variable, we will show it can also be applied when a function is separable. Specifically, under the separability Assumption \ref{asm:sep} below, we can construct the best uniform affine approximation by determining the best affine approximation on an interval for each component function. 

\begin{assumption}
\label{asm:sep}
Assume that $f: \mathbb{R}^{m \times n}\rightarrow \mathbb{R}$ can be \textbf{separated} into a sum of component functions, i.e., $f(X) = \sum_{i} f_{ij}(X_{ij})$, where each $f_{ij}:\mathbb{R} \rightarrow \mathbb{R}$.
\end{assumption}

\begin{theorem}
\label{thm:sep_defn}
Suppose $f: \mathbb{R}^{m \times n} \rightarrow \mathbb{R}$ is a continuous function that satisfies Assumption \ref{asm:sep}. For a given $X^{(k)} \in \mathbb{R}^{m \times n}$ and $\tau > 0$, if $\ell_{ij}(Y_{ij})$ is the Chebyshev polynomial of degree 1 for $f_{ij}$ over the interval $[X_{ij} - \tau, X_{ij} + \tau]$, then the function $\ell(Y)  = \sum_{i=1}^{m}\sum_{j=1}^{n}\ell_{ij}(Y_{ij})$ is the uniform affine approximation to $f$.
\end{theorem}

Thus, under Assumption \ref{asm:sep}, it suffices to find the Chebyshev polynomials for the component functions. Furthermore, if each $f_{ij}$ is convex, then there exists a closed form solution for a linear Chebyshev polynomial \cite{davis1975interpolation}.

\section{FWUA and Convergence}
Using the uniform affine approximation, we propose a FW variant with Uniform Approximations (FWUA), which is described in Algorithm \ref{alg:fwua}. The function \texttt{update\_tau} will be described in full in Section \ref{sec:update_tau}, where a specific update rule for $\tau$ will be required to guarantee convergence.

\begin{algorithm}  
    \caption{Frank-Wolfe with Uniform Approximations (\textbf{FWUA})}
  \label{alg:fwua}
  \begin{algorithmic}[1]
  \REQUIRE
  $f$: A function satisfying Assumptions \ref{asm:sep} and \ref{asm:comp}\\
  $\mathcal{D}$: A convex and compact subset of $\mathbb{R}^{m \times n}$\\
  $\epsilon$: Approximation threshold\\
  $K$: Max iteration count
  \STATE Let $X^{(0)} \in \mathcal{D}$.
  \STATE Let $\tau^{0} \gets \mathbf{diam}(\mathcal{D})$.
  \FOR{$k = 0..K$}
  \STATE $\alpha^{(k)} \gets \frac{2}{k+2}$
  \STATE $\tau^{(k)} \gets \texttt{update\_tau}(k, \tau^{(k-1)}, \epsilon,f )$
  \STATE $(\xi^{(k)}, b^{(k)}) \gets \displaystyle\argmin_{\xi, b} \max_{Y \in \bar{\mathcal{B}}_{\infty}(X^{(k)}, \tau^{(k)})} \abs{f(Y) - b - \langle Y - X, \xi \rangle}$
  \STATE $\displaystyle S^{(k)} \gets \argmin_{S} b^{(k)} + \langle S - X^{(k)}, \xi_{k} \rangle$
  \STATE $X^{(k+1)} \gets X^{(k)} + \alpha^{(k)} (S^{(k)} - X^{(k)})$.
  \ENDFOR 
\end{algorithmic}
\end{algorithm}

To establish convergence, subsequently we make the following assumptions.

\begin{assumption}
\label{asm:comp}
Assume that $f$ satisfies Assumption \ref{asm:sep}.
In addition, each component function $f_{ij}$ has the following properties:
\begin{enumerate}
    \item [(a)] $f_{ij}$ is a convex, $L_{ij}$-Lipschitz continuous function.
    \item [(b)] $f_{ij}$ is not differentiable on at most a finite set.
    \item[(c)] If $f_{ij}$ is differentiable at $a$, then it is also twice differentiable at $a$.
\end{enumerate}
\end{assumption}

While the above set of assumptions appears restrictive, our main goal in this work is to efficiently solve \eqref{RMC} in the context of trace-norm constrained matrix estimation problem which has a combination of $\ell_{1}$ and $\ell_{2}$ loss/regularization, for which these assumptions are satisfied. 

\begin{definition}
\label{def:slope_func}
Let $f$ be satisfy Assumption \ref{asm:sep} and \ref{asm:comp}. For a given $\tau>0$,  we define the \textbf{uniform slope function}, $m_{ij}(X_{ij},\tau) :\mathbb{R}\times\mathbb{R}^{+} \rightarrow \mathbb{R}$, which is the slope of the uniform affine approximation for $f_{ij}$ on the interval $[X_{ij} - \tau, X_{ij} + \tau]$.
\end{definition}
The  uniform slope function can be viewed as a surrogate for the gradient which does not rely on differentiability of $f$. 
\begin{theorem}
\label{thm:unif_conv}
Let $f:\mathbb{R}^{m \times n} \rightarrow \mathbb{R}$ be a function that satisfies Assumption \ref{asm:sep} and \ref{asm:comp} and $\mathcal{D}$ be a convex and compact set. Given $\tau>0$, let  $m_{ij}(X_{ij}, \tau)$ be the  uniform slope function for $f_{ij}$ at $X_{ij}$, and let $X_{\min} = \min\{X_{ij} | X = (X_{1,1},...,X_{ij},...,X_{mn})^{\top} \in \mathcal{D}\}$. Let 
 \begin{equation*}
\hat{f}(X, \tau) \coloneqq \sum_{i=1}^{m}\sum_{j=1}^{n}\int_{X_{\min}}^{X_{ij}} m_{ij}(x, \tau)dx + f_{ij}(X_{\min})
    \end{equation*}
Then the following statements hold:
\begin{enumerate}
    \item [(a)]$\hat{f}(X, \tau)$ is convex in $X$,
        \item [(b)] $\nabla_{X}\hat{f}(X, \tau)$ is $L/\tau$-Lipschitz continuous w.r.t. the $\ell_{\infty}$-norm, where $L$ is the max. Lipschitz constant of all $f_{ij}$,
    \item [(c)] $\max_{X \in \mathcal{D}}\abs*{\hat{f}(X, \tau) - f(X)} \leq mn(M + 1)D\Delta_{f}\tau$, where,    
  \begin{align}
  \begin{split}
  \label{eqn:constants}
    M &\coloneqq \text{the maximum number of points}\\
    &\phantom{{}={}}\text{any $f_{ij}$ is not differentiable at }\\
D &\coloneqq \mathbf{diam}(\mathcal{D}), \\
\Delta_{f} &\coloneqq \max\biggl\{\max_{\substack{i,j\\X,Y \in \mathcal{D}\\ f''_{ij} \text{ exists at } X_{ij}\text{ and }Y_{ij}}}\abs*{\frac{f_{ij}''(X_{ij}) - f_{ij}''(Y_{ij})}{2}},\\ &\phantom{{}\max\Biggr\{{}}\max_{\substack{i,j\\X_{ij}, Y_{ij} \in [X_{\min} - \tau, X_{ij} + \tau]\\ g_{ij} \in \partial f_{ij}(X_{ij})\\ h_{ij} \in \partial f_{ij}(Y_{ij})}}2\abs{g_{ij} - h_{ij}}\biggr\}
\end{split}
\end{align}
\end{enumerate}
\end{theorem}

Suppose the neighborhood size at iteration $k$ is given by $\tau^{(k)}$. Theorem \ref{thm:unif_conv} states that as $\tau^{(k)} \rightarrow 0$, the sequence of uniform affine approximations generated by the FWUA algorithm uniformly converges to the original objective. In particular, given a sequence of neighborhood sizes $\{\tau^{(k)}\}$, we consider the sequence of smooth approximations given by
\begin{equation}
\label{eqn:smooth_def}
\begin{split}
   \hat{f}^{(k)}(X) = \sum_{i=1}^{m}\sum_{j=1}^{n}\int_{X_{\min}}^{X_{ij}} m_{ij}(x, \tau^{(k)})dx + f_{ij}(X_{\min}).
\end{split}
\end{equation}
Since $\hat{f}^{(k)}$ is differentiable with a $\frac{L}{\tau^{(k)}}$-Lipschitz gradient, it can be shown that the curvature constant for $\hat{f}^{(k)}$  is bounded,
\begin{equation*}
C_{\hat{f}^{(k)}} \leq \frac{L}{\tau^{(k)}}\mathbf{diam}(\mathcal{D})^{2}
\end{equation*}
see, e.g., \cite{thesisMJ}. Thus, we can leverage standard Frank-Wolfe convergence argument while maintaining an upper bound on the approximation quality of the solution. 

To make these concepts concrete, consider $f(X) = \norm{X}_{1}$ and $\mathcal{D}$ is the trace norm ball of radius $\delta$. Then,
\begin{equation*}
    \hat{f}(X, \tau) = \sum_{ij} \left(\frac{X_{ij}^{2}}{2\tau}  + \frac{\tau}{2} \right)\cdot\mathbbm{1}_{\abs{X_{ij}} <  \tau} + \abs{X_{ij}}\cdot\mathbbm{1}_{\abs{X_{ij}} \geq \tau}
\end{equation*}
where $\mathbbm{1}$ is the indicator function.
Figure \ref{fig:tilde_f} illustrates the component functions.

\begin{figure}[t]
\centering
\includegraphics[height=0.14\textheight]{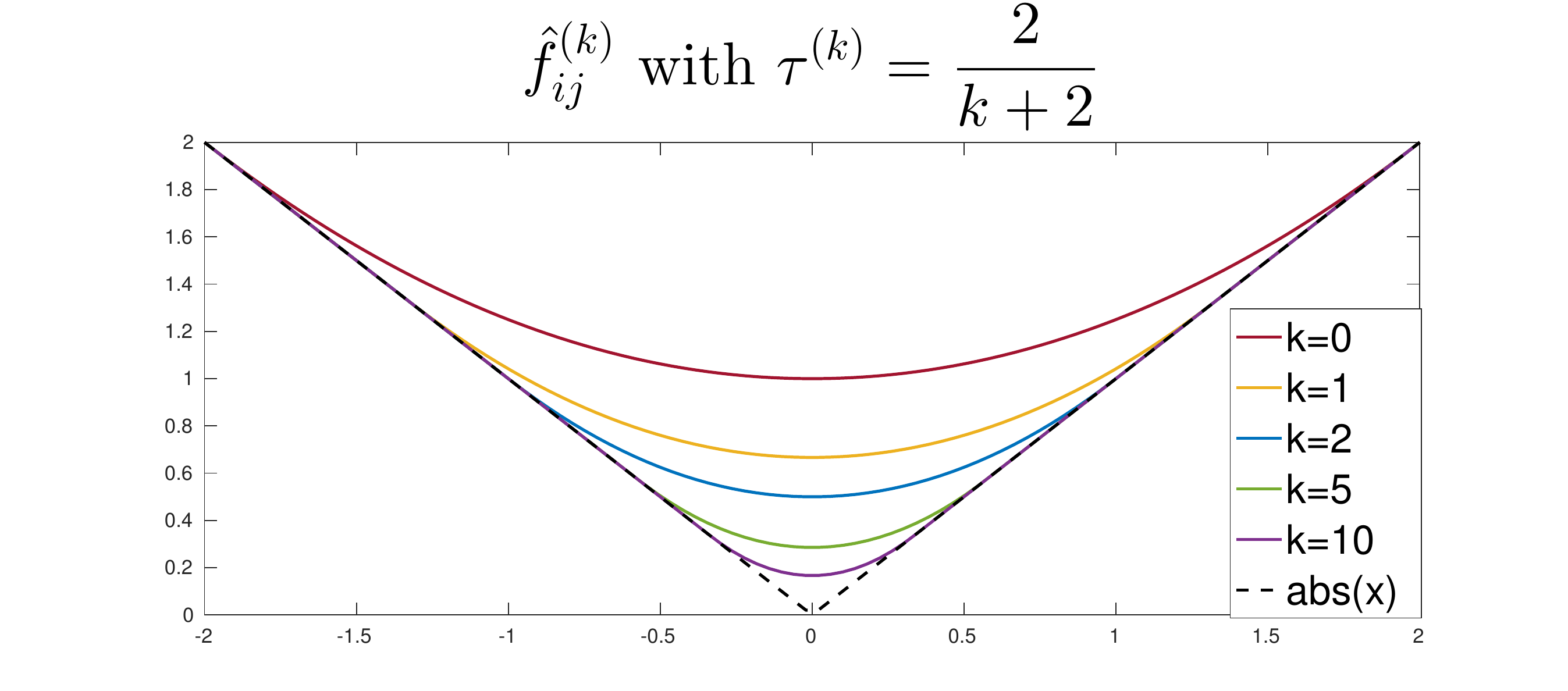}
\caption{The component functions $\hat{f}^{(k)}_{ij}$ with $\tau^{(k)} = 2/(k+2)$ and varying $k$ for $f(X) = \norm{X}_{1}$.}
\label{fig:tilde_f}
\end{figure}

When $f(X) = \norm{X}_{1}$, its uniform affine approximation has an attractive property that $\hat{f} \geq f$. In general, this property may not hold. However from the uniform error bound, there always exists some constant $N^{(k)} \in [0, n(M + 1)D\Delta_{f}\tau^{(k)}]$ such that $\hat{f}^{(k)} + N^{(k)} \geq f$, and the function $\hat{f}^{(k)} + N^{(k)}$ has all the properties listed in Theorem \ref{thm:unif_conv}, except the error bound in \eqref{eqn:constants} becomes $\norm{\hat{f}^{(k)} - f}_{\infty} \leq 2mn(M + 1)D\Delta_{f}\tau^{(k)}$. Thus, we redefine the sequence of approximations as follows.

\begin{definition}
Let $f: \mathbb{R}^{m \times n} \rightarrow \mathbb{R}$ be a function that satisfies Assumptions \ref{asm:sep} and \ref{asm:comp}. The sequence of \textbf{FWUA smooth approximations} are given by
\begin{equation}
\label{def:smooth_apprx_over}
   \tilde{f}^{(k)}(X) = \sum_{i,j}\int_{X_{\min}}^{X_{ij}} m_{ij}(x, \tau^{(k)})dx + f_{ij}(X_{\min}) + N^{(k)}
\end{equation}
where $N^{(k)} \geq 0$ is the smallest number such that $\tilde{f}^{(k)} \geq f$.
\end{definition}

\subsection{Update $\tau$}
\label{sec:update_tau}
The role of $\tau^{(k)}$ in the algorithm will be to reflect the maximum deviation in any component of $X^{(k)}$ after a FW step. That is, we are interested in bounding the quantity $\norm{X^{(k+1)} - X^{(k)}}_{\infty}$. Since the FW steps take the form,
\begin{equation*}
    X^{(k+1)} = X^{(k)} + \alpha^{(k)}(S^{(k)}- X^{(k)})
\end{equation*}
for some $S \in \mathcal{D}$, this leads to a straightforward bound of
\begin{equation*}
    \norm{X^{(k+1)} - X^{(k)}}_{\infty} \leq \alpha^{(k)}\mathbf{diam}(\mathcal{D}).
\end{equation*}
However, setting $\tau^{(k)} = \alpha^{(k)}\mathbf{diam}(\mathcal{D})$ can be overly conservative and lead to poor local approximations. Ideally, we would set $\tau^{(k)} = \norm{S^{(k)} - X^{(k)}}_{\infty}$ to have the smallest value for $\tau^{(k)}$ which ensures that $S^{(k)}$ is in the neighborhood of interest. This is not possible since $S^{(k)}$ depends on the choice of $\tau^{(k)}$. In implementation, we estimate $\tau^{(k)}$ by the previous FW steps and consider the update,
\begin{equation}
\label{eqn:update_rule}
    \tau^{(k+1)} \gets \alpha^{(k)}\max_{j \in \{0,...,4\}}\norm{X^{(k-j)} - S^{(k-j)}}_{\infty}.
\end{equation}

Another concern with updating $\tau^{(k)}$ is that we cannot allow the function $\tilde{f}^{(k)}$ to become arbitrarily close to $f$, since the Lipschitz constant for $\tilde{f}^{(k)}$ can also grow arbitrarily large given a nonsmooth $f$. However, if only an $\epsilon$-accurate solution is desired with $\epsilon$ specified a priori, one can stop refining the approximation $\tilde{f}$ at some iteration $k'$ since there exists an explicit upper bound on the approximation error. The FWUA algorithm can then proceed as a standard FW algorithm on the smoothed function defined by iteration $k'$.

Specifically, the uniform error bound from Theorem \ref{thm:unif_conv} and a step size of $\alpha^{(k)} = 2/(k+2)$ guarantees that
\begin{equation*}
    \frac{4mn(M+1)D^{2}\Delta_{f}}{k+2} \leq \frac{\epsilon}{2} \text{ when } k \geq \frac{8mn(M+1)D^{2}\Delta_{f}}{\epsilon} - 2.
\end{equation*}
At iteration $k' = \frac{8mn(M+1)D^{2}\Delta_{f}}{\epsilon} - 1$, we stop refining the neighborhoods and $\tau^{(k)} = \tau^{(k')}$ for all $k \geq k'$. The \texttt{update\_tau} function is be formalized in Algorithm \ref{alg:updatetau}.

\begin{algorithm}[h!]  
    \caption{\texttt{update\_tau}}
  \label{alg:updatetau}
  \small
  \begin{algorithmic}[1]  
  \REQUIRE
  $k$: Iteration number\\
  $\tau^{(k-1)}$: Previous neighborhood size\\
  $\epsilon$: accuracy tolerance\\
  $f$: Original function to optimize\\  
  $\{X^{(i)}\}_{i=0}^{k-1}$: Sequence of previous Frank-Wolfe iterates\\
  $\{S^{(i)}\}_{i=0}^{k-1}$: Solutions to the Frank-Wolfe linear subproblems of previous iterates\\
  $[m,n, M, D, \Delta_{f}] \gets$ parameters of $f$ as described in \eqref{eqn:constants}.
  \STATE $k' \gets \frac{8mn(M+1)D^{2}\Delta_{f}}{\epsilon} -1$.
  \IF{$k > k'$}
  \STATE \texttt{return} $\tau^{(k-1)}$
  \ELSE
  \STATE \texttt{return} $\tau^{(k)} \gets \frac{2\max_{j \in \{0,...,4\}}\norm{X^{(k-j)} - S^{(k-j)}}_{\infty}}{k+2}$
  \ENDIF
\end{algorithmic}
\end{algorithm}

The rationale is that once the approximation is sufficiently accurate, i.e., the approximation error is no bigger than $\epsilon/2$, then standard Frank-Wolfe analysis can be applied to bound the suboptimality of the smoothed problem defined by iteration $k'$ by $\epsilon/2$. We formalize this statement in Theorem \ref{thm:final_conv}.

\begin{theorem}
\label{thm:final_conv}
Let $f:\mathbb{R}^{m \times n} \rightarrow \mathbb{R}$ satisfy Assumptions \ref{asm:sep} and \ref{asm:comp}, $X^{*} \in \argmin_{X\in \mathcal{D}} f(X)$, and let $M, D, \Delta_{f}$ be the constants defined in \eqref{eqn:constants}. Then for any $\epsilon > 0$, the iterates $\{X^{(k)}\}$ of Algorithm \ref{alg:fwua} using $\alpha^{(k)} = 2/(k+2)$ satisfy
\begin{equation*}
f(X^{(k)}) - f(X^{*}) < \epsilon
\end{equation*}
when
\begin{equation*}
    k \geq k' + \frac{8C_{\tilde{f}^{(k')}}}{\epsilon} \text{ with }  k' = \frac{8mn(M+1)D^{2}\Delta_{f}}{\epsilon}.
\end{equation*}
\end{theorem}

\section{Experimental Results}
\subsection{Sparse and Low-Rank Structure}
To highlight benefits of the proposed FWUA, we first compare it against other state-of-the-art solvers for the problem,
\begin{equation*}
\label{eqn:sp_p_lr}
\min_{X: \norm{X}_{\tr} \leq \delta} \norm{P_{\Omega}(X - Y)}_{F}^{2} + \lambda_1 \norm{X}_{1}.
\end{equation*}
where $Y$ is the given data, $\Omega = \{(i,j)\}$ is the set of observed indices, and $P_{\Omega}(\cdot)$ projects the loss onto $\Omega$. For all experiments, the FW based methods terminate after 1000 iterations, and all other methods use default stopping criteria suggested by the authors. 

We compare FWUA with  GenFB \cite{richard2012estimation} , HCGS \cite{argyriou2014hybrid}, and SCCG \cite{pierucci2014}. For each  problem instance, the same  $\lambda_1$ value  is used by all methods and this value is tuned, by searching over a grid of parameter values, to yield the best test performance for GenFB. The bound $\delta$ for the trace norm, is then set to the trace norm of the solution given by GenFB. For  SCCG, the smoothing parameter $\mu$ is additionally tuned to yield the smallest average objective value. HCGS sets $\beta^{(k)} = 1/\sqrt{k+1}$ as suggeseted by the authors. We also compare the limiting behavior for SCCG when $\mu = 0$. This corresponds to a specific subgradient, denoted as SCCG (SG).\footnote{The final parameters used are omitted due to space, but will appear on the arXiv.}

\subsubsection{Sparse Covariance Estimation}
We follow the synthetic experiments described in \cite{richard2012estimation}, where the goal is to recover a block diagonal matrix. We consider square matrices where $n = 750:250:2000$ (here we use MATLAB notation). The true underlying matrix is generated with 5 blocks, where the entries are i.i.d. and uniformly sampled from $[-1,1]$. Gaussian noise, $\mathcal{N}(0, \sigma^{2})$ is then added with $\sigma^{2} = 0.2$. For this experiment, all entries in $\Omega$ are observed.

In Figure \ref{fig:cov_conv}, we only present the convergence results for $n = 2000$ due to space, but the patterns are similar throughout. We remark that since the GenFB algorithm is a regularized algorithm, the intermediate iterates are not feasible for the constrained problem used for FW. Only the performance of the solution at convergence is compared.
\begin{figure}[!h]
\centering
    \begin{subfigure}[b]{0.22\textwidth}
        \centering
        \includegraphics[height=0.15\textheight]{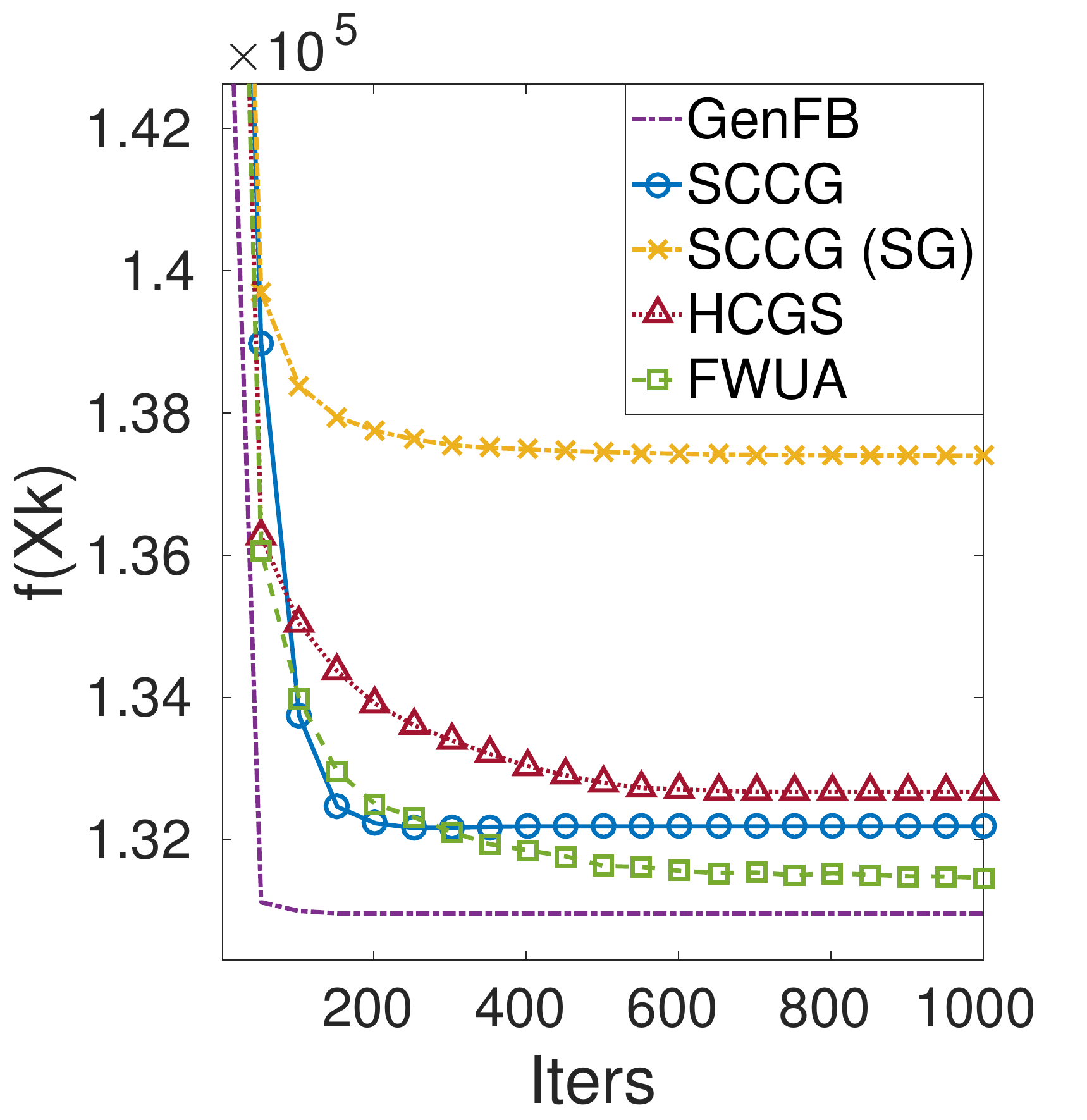}        
    \end{subfigure}%
    ~
    \begin{subfigure}[b]{0.22\textwidth}
        \centering
        \includegraphics[height=0.15\textheight]{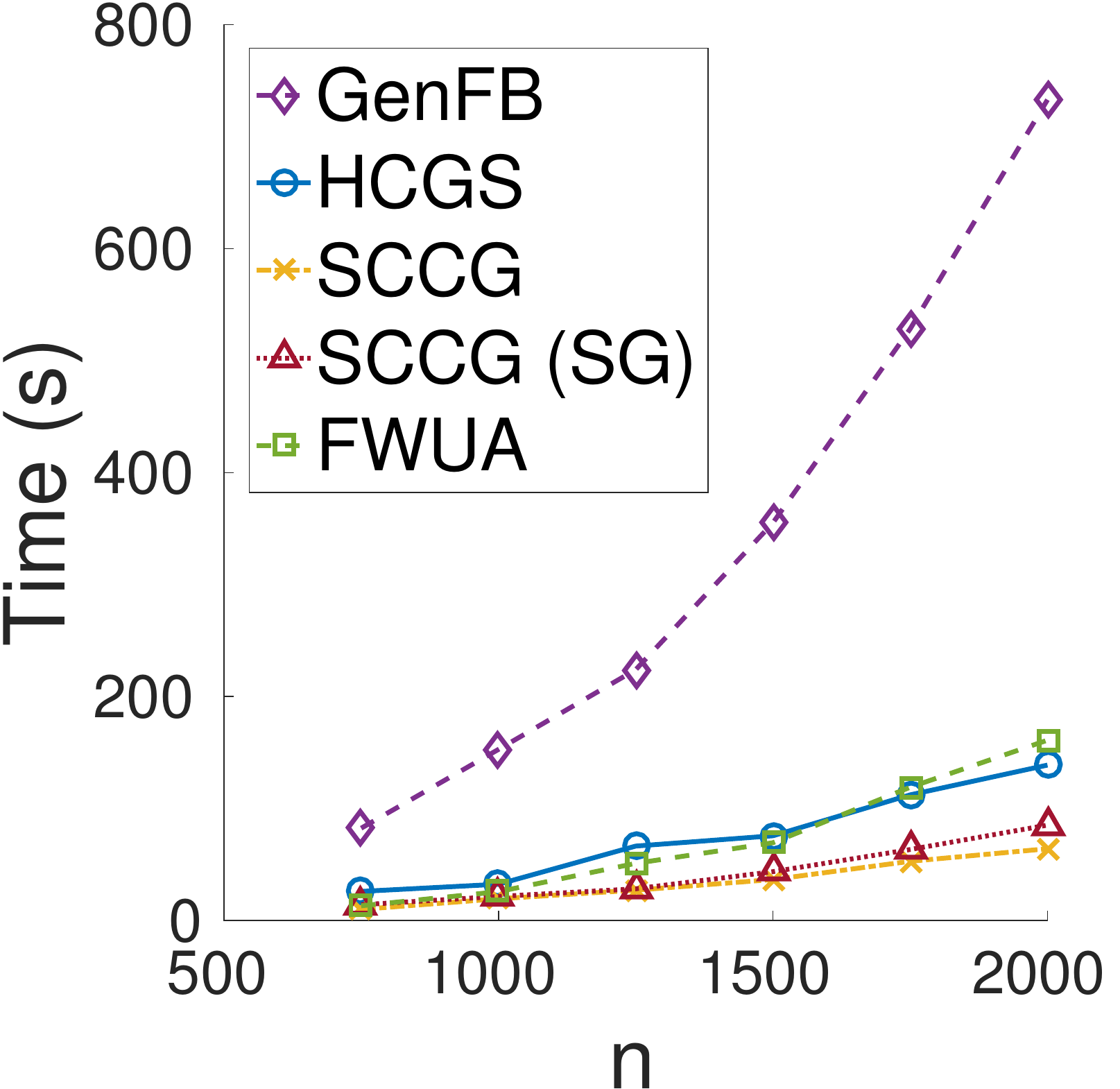}       
    \end{subfigure}%
    \caption{Sparse covariance estimation.}
    \label{fig:cov_conv}
\end{figure}

\begin{figure}[ht]
\centering
\includegraphics[width=0.5\textwidth]{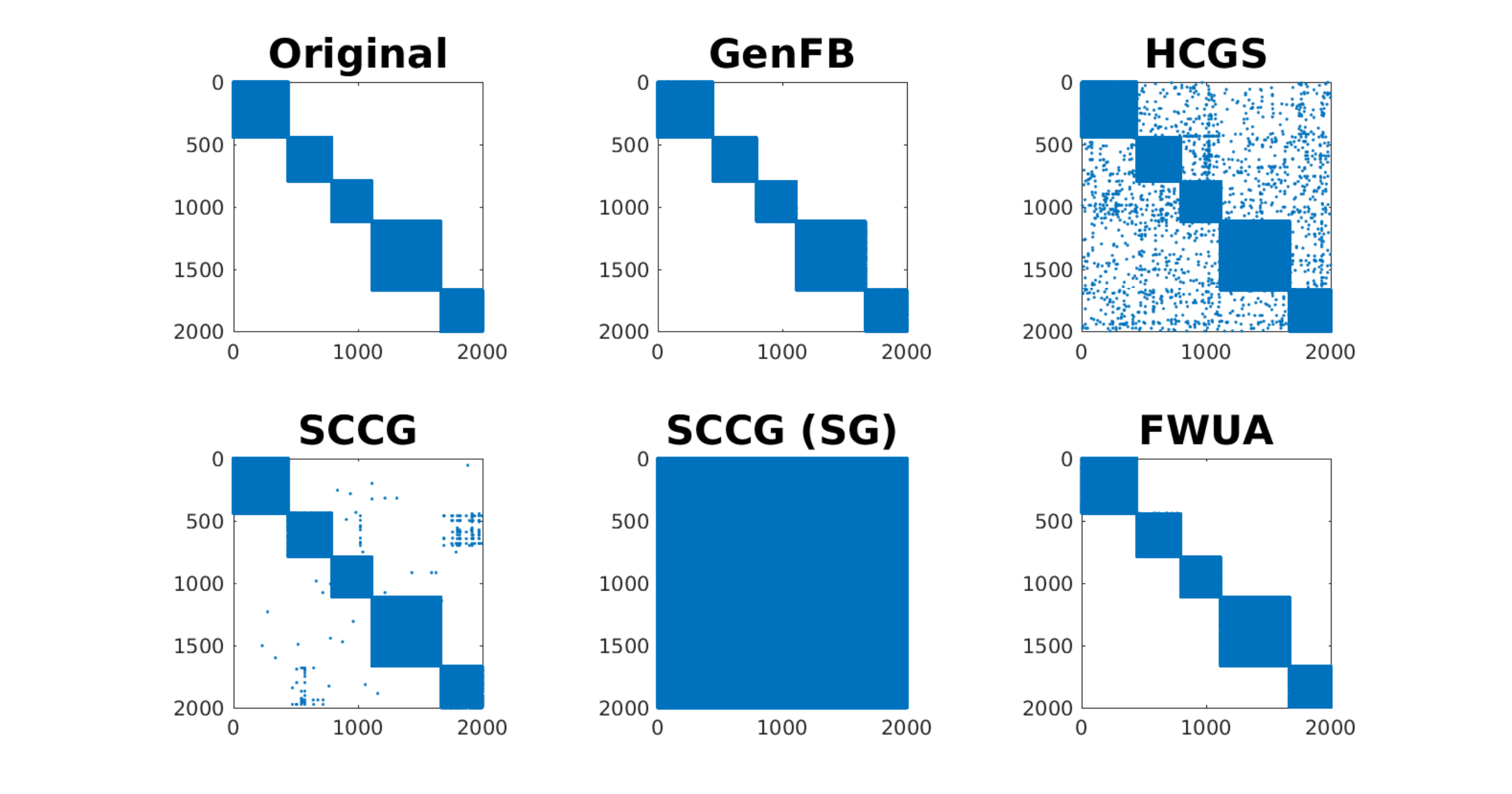}   
\caption{An example sparsity pattern at convergence, values thresholded at $0.01\norm{X}_{\infty}$.}
\label{fig:sparse}
\end{figure}

\subsubsection{Graph Link Prediction}
Next  we consider  predicting links in a noisy social graph. The input data is a matrix corresponding to an undirected graph, where the entry $A_{ij} = 1$ indicates that user $i$ and $j$ are friends and $A_{ij} = 0$ otherwise. We consider the Facebook dataset from \cite{snapnets} which consists of a graph with 4,039 nodes and 88,234 edges, and assume 50\% of the entries are observed. The goal is to recover the remaining edges in the graph. Additionally, each entry $A_{ij}$ is flipped with probability $\sigma  \in \{0, 0.05, 0.1\}$, potentially removing or adding labels to the graph. We report the AUC performance measure of the link prediction on the remaining entries of the graph as well as the average CPU time over 5 random initializations summarized in Table \ref{tbl:graph} across all levels of $\sigma$.

\renewcommand{\arraystretch}{1.1}
\begin{table}[!t]
\centering
\scriptsize
\begin{tabular}{c|lrrrrr}
\specialrule{.2em}{.1em}{.1em}
$\sigma$ & & GenFB & SCCG & SCCG (SG) & HCGS & FWUA\\
\specialrule{.2em}{.1em}{.1em}
  \multirow{2}{*}{0} & AUC & 0.968 & 0.949 & 0.873 & 0.868 & \textbf{0.972}\\ 
 & Time (s) & 3746.30 & 1127.91 & 288.46 & \textbf{263.27} & 449.89 \\
   \hline
   \multirow{2}{*}{0.05} & AUC & \textbf{0.829} & 0.820 & 0.799 & 0.806 & \textbf{0.829}\\ 
 & Time (s) & 4024.96 & \textbf{332.23} & 378.68 & 407.87 & 408.37 \\
   \hline
   \multirow{2}{*}{0.10} & AUC & 0.715 & 0.708 & 0.707 & 0.708 & \textbf{0.732}\\
  & Time (s) & 4006.92 & \textbf{363.18} & 403.27 & 420.22 & 481.82 \\
 \specialrule{.2em}{.1em}{.1em}
\end{tabular}
\caption{Graph link prediction averaged over 5 random initializations for the Facebook dataset. Top performers are bolded.}
\label{tbl:graph}
\end{table}

\subsubsection{Discussion}
For both applications, the results agree with our initial intuition that FWUA can improve the performance of the FW variants while scaling much better than the GenFB algorithm. We observe in the covariance plots in Figure \ref{fig:sparse}, the sparsity patterns for HCGS and SCCG are much noisier than FWUA and in Table \ref{tbl:graph}, the AUC for SCCG and HCGS methods are lower than GenFB and FWUA. 

For SCCG, tuning the smoothing parameter $\mu$ yields theoretical tradeoffs between accuracy and convergence rates.  Initially, we expected this tradeoff to be smooth, where gradually decreasing $\mu$ gradually worsened the rate of improvement per iteration. In return we expected small values of $\mu$ to yield better final solutions. Instead, we observed that the algorithm will make no progress for many iterations if $\mu$ is too small, as highlighted in Figure \ref{fig:mu_evo}. We see that the delay observed increases as $\mu$ decreases. This seems to give support to the hypothesis for FWUA where the approximation quality must be closely related to the step size since it appears that SCCG can only make progress once the step sizes, following the step-size schedule $2/(k+2)$, becomes sufficiently small. Since HCGS and SCCG do not factor in step size into the smoothing schedule, we observe that empirically, the solutions returned from these methods are not as competitive as the FWUA or GenFB algorithm.

\begin{figure}[!h]
\centering
\includegraphics[width=0.4\textwidth]{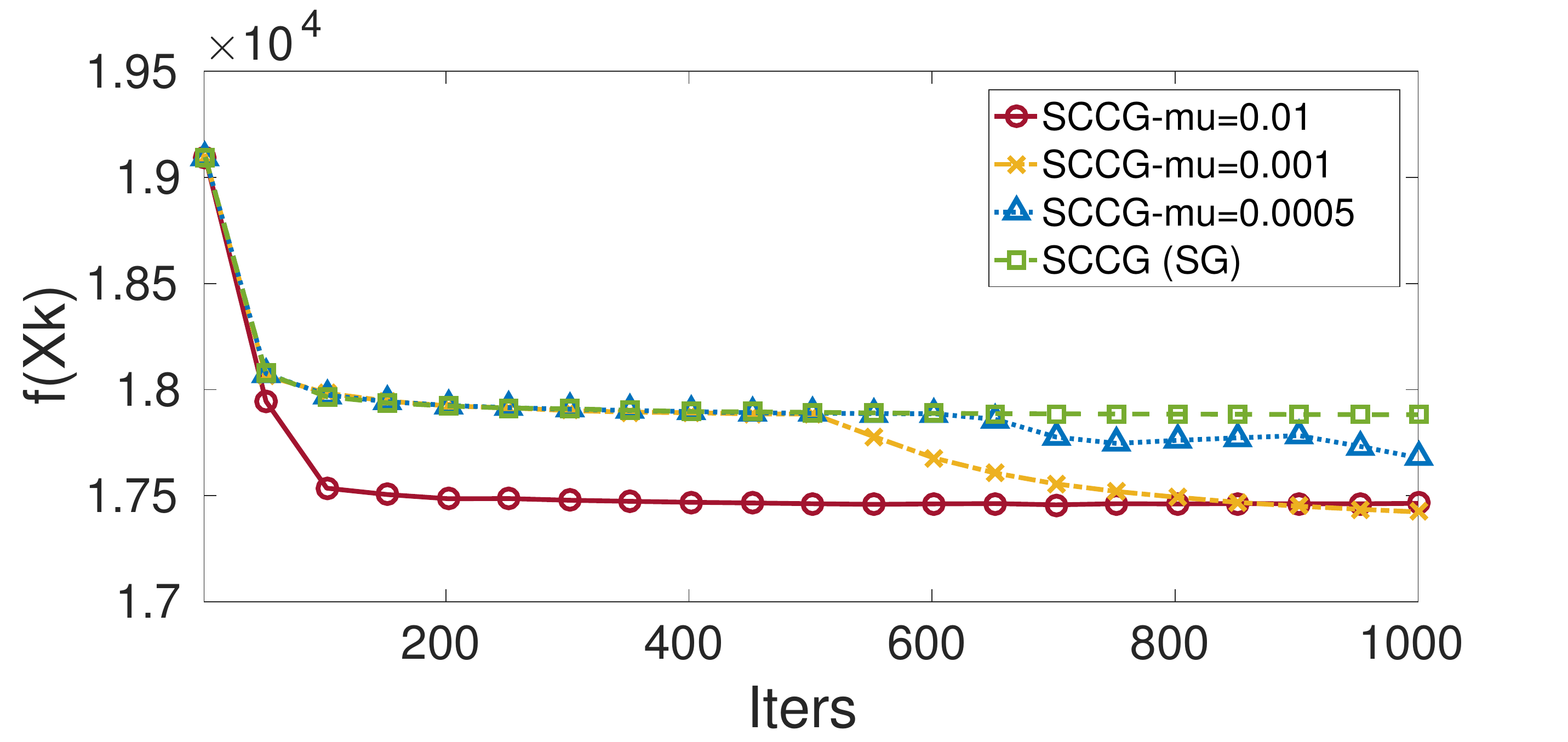}   
\caption{One trial with varying $\mu$ for sparse covariance estimation with $n = 750$. Here we see large delays as $\mu$ becomes smaller, almost making no progress for 700 iterations when $\mu = 0.0005$. }
\label{fig:mu_evo}
\end{figure}

\subsection{$\ell_{1}$ Loss Matrix Completion}
The last experiment we consider is matrix completion with an $\ell_{1}$ loss function on the MovieLens datasets\footnote{\url{https://grouplens.org/datasets/movielens/}}. Here, we consider  the objective function below
\begin{equation}
    f(X) = \norm{P_{\Omega}(X - Y)}_{1} + \lambda\norm{P_{\Omega^{c}}(X)}_{F}^{2}
\end{equation}
which is proposed in \cite{cambier2016robust} for robustness to outliers. Here the regularization penalizes  entries in the complement of $\Omega$, potentially preventing overfitting. 

We compare with the Robust Low-Rank Matrix Completion (RLRMC) algorithm proposed in \cite{cambier2016robust}, which solves a nonconvex fixed rank problem by the smoothing $\ell_{1}$ term. We additionally compare to the Greedy Low-Rank Learning (GLRL) algorithm proposed in \cite{yao2016greedy}, which greedily updates the solution with low-rank solutions found by computing the truncated SVD of a subgradient. For scalability, we utilize the Rank-Drop variant (RDFW)  proposed in \cite{cheung2017rank} for the FWUA algorithm, which empirically reduces the rank and computation time of the FW algorithm. 

\begin{table}[!h]
    \centering
    \scriptsize
    \begin{tabular}{llrrr}
    \specialrule{.2em}{.1em}{.1em}
    Dataset & & RLRMC & GLRL & FWUA\\
    \specialrule{.2em}{.1em}{.1em}
      \multirow{2}{*}{ML-100k} & RMSE & 0.892 & 0.935 & \textbf{0.876}\\ 
      & Time (s) & \textbf{10.62} & 43.56 & 65.39\\
       \hline
       \multirow{2}{*}{ML-1M} & RMSE & 0.817 & 0.917 & \textbf{0.812}\\ 
     & Time (s) & \textbf{108.97} & 1,079.11 & 862.15\\
       \hline
       \multirow{2}{*}{ML-10M} & RMSE & 0.810 & 0.901 & \textbf{0.801}\\
       & Time (s) & \textbf{2,197.20} & 26,473.89 & 6,830.31\\     
     \specialrule{.2em}{.1em}{.1em}
    \end{tabular}
    \caption[Nonsmooth matrix completion results]{Low-rank Matrix Completion averaged over 5 random initializations. Best performers are bolded.}
    \label{tbl:rmc}
    \end{table}
\subsubsection{Discussion}
We observe that FWUA performs better than both RLRMC and GLRL in terms of out of sample RMSE, but RLRMC is much faster. This is not surprising since RLRMC is a nonconvex fixed rank model. We observe that the performance of RLRMC is very sensitive to both the rank and $\lambda$ parameters, requiring extensive parameter tuning to find reasonable results. Thus, FWUA can be an attractive alternative when a good estimate of the true rank is not known a priori. We also note that for the large scale example, the number of singular values required in the truncated SVD used by the GLRL updates became very high, leading to scalability issues.

\section{Conclusion}
We propose a variant of the Frank-Wolfe algorithm for a nonsmooth objective, by replacing the linear FW subproblem with one defined by the Chebyshev uniform affine approximation. We show that for nonsmooth matrix estimation problems, this uniform approximation is easy to compute and allows for convergence analysis without assuming a bounded curvature constant. Experimentally we demonstrate that the FWUA algorithm can improve both speed and classification performance in a variety of sparse and low-rank learning tasks, while providing a viable convex alternative for $\ell_{1}$ loss matrix completion when little is known about the underlying data.
\bibliographystyle{named}
\bibliography{main}


\onecolumn
  \hsize\textwidth
  \linewidth\hsize {\centering
  {\Large\bf Supplementary Material \par}}

\subsection*{Proof of Theorem \ref{thm:sep_defn}}
\begin{proof}
Following the equioscillation Theorem \ref{thm:cheb}, there exists $\ell_{ij}$ such that,
\begin{equation}
\ell_{ij} = \argmin_{p \in \Pi_{1}} \max_{Y_{ij} \in [X^{(k)}_{ij} - \tau, X^{(k)}_{ij} + \tau]}\abs{f_{ij}(Y_{ij}) - p(Y_{ij})}.
\end{equation}
Let
\begin{equation}
b + \langle Y - X^{(k)}, \xi \rangle = \sum_{i=1}^{m}\sum_{j=1}^{n}\ell_{ij}(Y_{ij})
\end{equation}
Since $f(X)=\sum_{i=1}^{m}\sum_{j=1}^{n} f_{ij}(X_{ij})$, we have that,
\begin{align}
\label{eqn:sum_ub}
\begin{split}
&\phantom{{}{=}{}}\sum_{i=1}^{m}\sum_{j=1}^{n}\max_{Y_{ij} \in [X^{(k)}_{ij} -\tau, X^{(k)}_{ij} + \tau]} \abs{f_{ij}(Y_{ij}) - \ell_{ij}(Y_{ij})}\\
&\geq \max_{Y \in \bar{\mathcal{B}}_{\infty}(X^{(k)}, \tau)} \abs*{\sum_{i=1}^{m}\sum_{j=1}^{n}f(Y_{ij}) - \ell_{ij}(Y_{ij})}\\
&= \max_{Y \in \bar{\mathcal{B}}_{\infty}(X^{(k)}, \tau)} \abs*{f(Y) - b - \langle Y - X^{(k)}, \xi \rangle}.
\end{split}
\end{align}
By continuity of $f$, there exists some $\bar{Y}_{ij} \in [X^{(k)}_{ij} - \tau, X^{(k)}_{ij} + \tau]$ such that,
\begin{equation}
\label{eqn:ybar}
f_{ij}(\bar{Y}_{ij}) - \ell_{ij}(\bar{Y}_{ij})  = \max_{Y_{ij} \in [X^{(k)}_{ij} - \tau, X^{(k)}_{ij} + \tau]}\abs{f_{ij}(Y_{ij}) - \ell_{ij}(Y_{ij})}.
\end{equation}
where the nonnegativity of \eqref{eqn:ybar} is guaranteed by the equioscillation Theorem \ref{thm:cheb}.

Let $\bar{Y} = [\bar{Y}_{ij}]$. From $f_{ij}(\bar{Y}_{ij}) - \ell_{ij}(\bar{Y_{ij}})\geq 0$ and  $f(X)=\sum_{i=1}^{m}\sum_{j=1}^{n} f_{ij}(X_{ij})$,
\begin{align}
\label{eqn:sum_lb}
\begin{split}
&\phantom{{}{=}{}}\sum_{i=1}^{m}\sum_{j=1}^{n}\max_{Y_{ij} \in [X^{(k)}_{ij} - \tau, X^{(k)}_{ij} + \tau]} \abs{f_{ij}(Y_{ij}) - \ell_{ij}(Y_{ij})}\\
&= \sum_{i=1}^{m}\sum_{j=1}^{n} f_{ij}(\bar{Y}_{ij}) - \ell_{ij}(\bar{Y}_{ij})\\
&= f(\bar{Y}) + b + \langle \bar{Y} - X^{(k)}, \xi \rangle\\
&\leq \max_{Y \in \bar{\mathcal{B}}_{\infty}(X^{(k)}, \tau)} \abs*{f(Y) - b - \langle Y - X^{(k)}, \xi \rangle}.
\end{split}
\end{align}
Combining \eqref{eqn:sum_ub} and \eqref{eqn:sum_lb}, we have that,
\begin{align}
\begin{split}
   &\phantom{{}{=}{}}\sum_{i=1}^{m}\sum_{j=1}^{n}\max_{Y_{ij} \in [X^{(k)}_{ij} - \tau, X^{(k)}_{ij} + \tau]} \abs{f_{ij}(Y_{ij}) - \ell_{ij}(Y_{ij})}\\
   &= \max_{y \in \bar{\mathcal{B}}_{\infty}(X^{(k)}, \tau)} \abs*{f(Y) - b - \langle Y - X^{(k)}, \xi \rangle} 
\end{split}
\end{align}
\end{proof}

\subsection*{Proof of Theorem \ref{thm:unif_conv}}
Before we establish Theorem \ref{thm:unif_conv}, we require the following results.

\begin{lemma}[Convex Mean Value Theorem \cite{wegge1974mean}]
\label{lem:mvt}
If $f(X)$ is a closed proper convex function from $\mathbb{R}^{m \times n} \rightarrow \mathbb{R}$ for $X$ in a convex set $\mathcal{D} \subseteq \mathbb{R}^{m \times n}$, then $X_{0}$ and $X_{1} \in ri(\mathcal{D})$ implies that there exists $0 < t < 1$, and a subgradient $G \in \partial f(C)$, where $C = tX_{0} + (1-t)X_{1}$, such that $f(X_{1}) - f(X_{0}) = \langle X_{1} - X_{0}, G \rangle$. 
\end{lemma}

\begin{theorem}
\label{thm:aff_def}
Suppose that $f$ satisfies Assumption \ref{asm:sep} and each $f_{ij}$ is closed, proper, and convex. Then the best affine approximation as defined in Definition \ref{def:ue} to each $f_{ij}$ on the interval $[a,b]$ is given by,
\begin{equation*}
\ell_{ij}(x) = \frac{f_{ij}(c) + h_{ij}^{+}(c)}{2} + \frac{f_{i}(b) - f_{ij}(a)}{b - a}(x - c)
\end{equation*}
where
\begin{equation*}
h^{+}_{ij}(x) \coloneqq f_{ij}(a) + \frac{f_{ij}(b) - f_{ij}(a)}{b - a}(x - a)
\end{equation*}
and $c \in (a,b)$ is chosen to satisfy Lemma \ref{lem:mvt}, the convex mean value theorem \cite{wegge1974mean} for $f_{ij}$ on $[a,b]$.
\end{theorem}

\begin{proof}
The affine function $h^{+}_{ij}(x)$ defines the line that connects $(a,f_{ij}(a))$ to $(b, f_{ij}(b))$. Since $f_{ij}$ is convex, $f_{ij} \leq h^{+}_{ij}$ on $[a,b]$.

From Lemma \ref{lem:mvt},  there exists $c \in (a,b)$ such that
\[
\frac{f_{ij}(b)-f_{ij}(a)}{b-a} \in \partial f_{ij}(c)
\]
Define
\[
h^{-}_{ij}(x) = f_{ij}(c) + \frac{f_{ij}(b)-f_{ij}(a)}{b-a}(x-c)
\]
The function $h^{-}_{ij}(x)$ is the line tangent to $f_{ij}(x_{ij})$ at $x_{ij}=c$ and is parallel to $h^{+}_{ij}$. Since $f_{ij}$ is convex, $f_{ij} \geq h^{-}_{ij}$ on $[a,b]$.

By construction, $\ell_{ij}$ is a line parallel and equidistant to the lines $h^{+}_{ij}$ and $h^{-}_{ij}$. Thus, it is easy to verify that
\begin{equation}
f_{ij}(a) - \ell_{ij}(a) = -(f_{ij}(c) - \ell_{ij}(c)) = f_{ij}(b) - \ell_{ij}(b)
\end{equation}
satisfying the equioscillation property. Thus, $\ell_{ij}$ is the minimax affine approximation to $f_{ij}$ on $[a,b]$.
\end{proof}

\begin{lemma}
\label{lem:m_bound}
Let $f$ be a function that satisfies Assumption \ref{asm:sep} with convex $f_{ij}$ and let $m_{ij}(X_{ij},\tau)$ be the corresponding uniform slope function for $f_{ij}$. Then 
\begin{equation*}
m_{ij}(X_{ij}, \tau) = g_{ij}(c) 
\end{equation*}
for some $c \in (X_{ij} - \tau, X_{ij} + \tau)$ where $g_{ij}$ is a subgradient of $f_{ij}$.
\end{lemma}

\begin{proof}
From Theorem \ref{thm:aff_def}, the slope function has the form,
\begin{equation}
    m_{ij}(X_{ij}, \tau) = \frac{f_{ij}(X_{ij} + \tau) - f_{ij}(X_{ij} - \tau)}{2\tau}
\end{equation}
which is simply the slope of the secant line of $f_{ij}$ from $X_{ij} - \tau$ to $X_{ij} + \tau$. Thus, the desired result follows immediately from the convex mean value theorem.
\end{proof}

\begin{lemma}
\label{lem:cvg_to_grad}
Let $f$ be a function that satisfies Assumption \ref{asm:sep} and \ref{asm:comp}, $m_{ij}(x, \tau)$ be the corresponding uniform slope function for $f_{ij}$, and  $a\leq b$. We have that,
\begin{equation*}
\int_{a}^{b}\abs{m_{ij}(x, \tau) - f_{ij}'(x)}dx \leq (b-a)\max_{\substack{y,z \in (a - \tau, b + \tau)\\g \in \partial f_{ij}(y)\\ h \in \partial f_{ij}(z)}}\abs{g - h}
\end{equation*}
where $f'_{ij}(x) \in \partial f_{ij}(x)$ can be any subgradient of $f_{ij}$ at $x$.

If $f_{ij}$ is additionally twice differentiable on $(a-\tau,b+\tau)$, then,
\begin{equation*}
\int_{a}^{b} \abs{m_{ij}(x, \tau)  - f_{ij}'(x)}dx \leq (b-a)\tau\max_{y,z \in (a,b)}\abs*{\frac{f''(y) - f''(z)}{2}}.
\end{equation*}

\end{lemma}
\begin{proof}
Using intermediate value theorem for integrals, there exists $c \in (a,b)$ such that
\begin{equation*}
\int_{a}^{b} \abs{m_{ij}(x, \tau)  - f_{ij}'(x)}dx = (b-a)\abs{m_{ij}(c, \tau) - f_{ij}'(c)}.
\end{equation*}

From Theorem \ref{thm:aff_def}
\begin{equation}
\label{eqn:central_diff}
m_{ij}(c, \tau) = \frac{f_{ij}(c + \tau) - f_{ij}(c - \tau)}{2\tau}.
\end{equation}

From Lemma \ref{lem:m_bound}, we have that $m_{ij}(x, \tau) = g_{ij}(c)$ for some $c \in (x - \tau, x + \tau)$. Since $m_{ij}(x, \tau)$ and $f_{ij}'(x)$ are just specific subgradients on the evaluated on the interval $(x - \tau, x + \tau)$, we have
\begin{equation*}
\abs{m_{ij}(x, \tau) - f_{ij}'(x)} \leq\max_{\substack{y,z \in (a - \tau, b + \tau)\\g \in \partial f_{ij}(y)\\ h \in \partial f_{ij}(z)}}\abs{g - h}.
\end{equation*}

Following the Lagrange Remainder Theorem, if $f_{ij}$ is twice differentiable in $(c - \tau, c+ \tau)$, we have
\begin{align}
\label{eqn:lag_rem}
\begin{split}
f_{ij}(c + \tau) &= f_{ij}(c) + f_{ij}'(c)\tau + \frac{1}{2}f_{ij}''(d_{1})\tau^{2}\\
f_{ij}(c - \tau) &= f_{ij}(c) - f_{ij}'(c)\tau  + \frac{1}{2}f_{ij}''(d_{2})\tau^{2}
\end{split}
\end{align}
for some $d_{1} \in (c, c + \tau)$ and $d_{2} \in (c - \tau, c)$. 
Substituting \eqref{eqn:lag_rem} into \eqref{eqn:central_diff},
\begin{align*}
\begin{split}
\frac{f_{ij}(c + \tau) - f_{ij}(c - \tau)}{2\tau} &= \frac{1}{2}\left(\frac{f_{ij}(c + \tau) - f_{ij}(c)}{\tau} - \frac{f_{ij}(c - \tau) - f_{ij}(c)}{\tau}\right)\\
&= \frac{1}{2}\left(f_{ij}'(c) +\frac{1}{2}f_{ij}''(d_{1})\tau + f_{ij}'(c) - \frac{1}{2}f_{ij}''(d_{2})\tau\right)\\
&= f_{ij}'(c) + \frac{f_{ij}''(d_{1}) - f_{ij}''(d_{2})}{2}\tau
\end{split}
\end{align*}
This implies,
\begin{align*}
\int_{a}^{b}\abs{m_{ij}(x, \tau) - f_{ij}'(x)}dx &= (b-a)\abs{m_{ij}(c, \tau) - f_{ij}'(c)}\\
&\leq (b-a)\tau\abs*{\frac{f_{ij}''(d_{1}) - f_{ij}''(d_{2})}{2}}\\
&\leq (b-a)\tau\max_{y,z \in (a,b)}\abs*{\frac{f''(y) - f''(z)}{2}}
\end{align*}
and the result follows.
\end{proof}

\begin{proof} (Theorem \ref{thm:unif_conv})
\begin{enumerate}
    \item [(a)] For notational simplicity, we drop the dependency on $\tau$ for $\tilde{f}_{ij}$ and $m_{ij}$.

    We establish that $\tilde{f}$ is convex by showing that each $\tilde{f}_{ij}$ is convex. Since $\tilde{f}_{ij}$ is a differentiable function of one variable, $\tilde{f}_{ij}$ is convex if and only if  $\tilde{f}_{ij}'$ is nondecreasing in $X_{ij}$. We have that for any $h > 0$,
    \begin{align*}
    \tilde{f}_{ij}'(X_{ij} + h) - \tilde{f}_{ij}'(X_{ij}) &= m_{ij}(X_{ij} + h) - m_{ij}(X_{ij})\\
    &= \frac{f_{ij}(X_{ij} + h + \tau) - f_{ij}(X_{ij} + h - \tau)}{2\tau} \\
    &\phantom{{}{=}{}}- \frac{f_{ij}(X_{ij} + \tau) - f_{ij}(X_{ij} - \tau)}{2\tau}.
    \end{align*}
    
    Since $f_{ij}$ is convex, we have that the slope of any secant,
    \begin{equation*}
    S(X_{ij},Y_{ij}) = \frac{f_{ij}(X_{ij}) - f_{ij}(Y_{ij})}{X_{ij} - Y_{ij}}
    \end{equation*}
    is nondecreasing in either $X_{ij}$ or $Y_{ij}$ \cite{gordon2001real}.
    
    Thus,  $\tilde{f}_{ij}'$ is nondecreasing follows immediately since,
    \begin{align*}
    \frac{f_{ij}(X_{ij} + \tau) - f_{ij}(X_{ij} - \tau)}{2\tau} &\leq \frac{f_{ij}(X_{ij} + h + \tau) - f_{ij}(X_{ij} - \tau)}{2\tau + h} 
    \\&\leq \frac{f_{ij}(X_{ij} + h + \tau) - f_{ij}(X_{ij} + h - \tau)}{2\tau}.
    \end{align*}
    
    \item [(b)] From the definition of $\tilde{f}$,
    \begin{equation}
    \frac{\partial}{\partial X_{ij}} \tilde{f}(X, \tau) = m_{ij}(X_{ij}, \tau).
    \end{equation}
    To verify the Lipschitz condition, note that from Theorem \ref{thm:aff_def}, 
    \begin{equation}
    m_{ij}(X_{ij},\tau) = \frac{f_{ij}(X_{ij} + \tau) - f_{ij}(X_{ij} - \tau)}{2\tau}.
    \end{equation}
    For any $y, z \in \mathbb{R}$ and  $\tau > 0$, we have,
    \begin{align}
    \begin{split}
    \abs{m(z, \tau) - m(y,\tau)} &= \abs*{\frac{f_{ij}(z + \tau) - f_{ij}(z - \tau)}{2\tau} - \frac{f_{ij}(y + \tau) - f_{ij}(y - \tau)}{2\tau}} \\
    &\leq \frac{1}{2\tau}\abs{(f_{ij}(z + \tau) - f_{ij}(y + \tau))} + \abs{(f_{ij}(z - \tau) - f_{ij}(y - \tau)}\\
    &= \frac{1}{\tau}L_{ij}\abs{z - y} \text{ (since $f_{ij}$ is Lipschitz continuous)}
    \end{split}
    \end{align}
    Thus,
    \begin{align*}
    \norm*{\nabla_{X}\tilde{f}(Z, \tau) - \nabla_{X}\tilde{f}(Y,\tau)}_{\infty} &= \max_{i,j}~\abs*{m_{ij}(Z_{ij}, \tau) - m_{ij}(Y_{ij}, \tau)}\\
    &\leq \frac{L}{\tau}\norm{Z - Y}_{\infty},
    \end{align*}
    where $\norm{\cdot}_{\infty}$ is the component-wise maximum absolute value.
    
    \item [(c)] Note we can expand the maximum as follows,
    \begin{equation*}
    \max_{X \in \mathcal{D}} \abs*{\tilde{f}(X, \tau) - f(X)} = \max_{X \in \mathcal{D}} \abs*{\sum_{i=1}^{m}\sum_{j=1}^{n}\int_{X_{\min}}^{X_{ij}} m_{ij}(x, \tau)dx + f_{ij}(X_{\min}) - f_{ij}(X_{ij})}.
    \end{equation*}
    Suppose $f_{ij}$ is not differentiable only at the points $c_{1}, c_{2},..., c_{M_{i}}$. Partition the interval $[X_{\min}, X_{ij}]$ as follows. Let $\mathcal{A}_{ij}$ be a collection of intervals,
    \begin{align*}
    \mathcal{A}_{ij} \coloneqq \{[\alpha_{t}, \beta_{t}]: \alpha_{t} &= \max\{c_{t} - \tau, c_{t-1} + \tau, X_{\min}\}, \text{ and }
    \\ \beta_{t} &= \min\{c_{t} + \tau, c_{t+1} - \tau, X_{ij}\},
    \end{align*}
    with $c_{0} = X_{\min}$ and $c_{M_{i}+1} = X_{ij}$. We can interpret each interval $[\alpha_{t}, \beta_{t}]$ as a neighborhood around each point of nondifferentiability with length at most $2\tau$. Note that the intervals do not overlap except possibly at the endpoints, and do not extend past the interval $[X_{\min}, X_{ij}]$. 
    
    Let $\mathcal{B}_{ij}= \bigcup (b_{t},b_{t+1})$ be the minimal set of intervals such that  $\mathcal{B}_{ij}=[X_{\min}, X_{ij}] \setminus \mathcal{A}_{ij}$. Thus, $\mathcal{A}_{ij} \cup \mathcal{B}_{ij}$ covers the interval $[X_{\min}, X_{ij}]$ and we have that for every interval in $\mathcal{B}_{ij}$, $f_{ij}$ is differentiable, and hence twice differentiable from our assumptions.
    
    Then we can write
    $\int_{X_{\min}}^{X_{ij}}m_{ij}(x,\tau)dx$ as a sum of integrals over intervals from $\mathcal{A}_{ij}$ and $\mathcal{B}_{ij}$,
    \begin{equation*}
    \int_{X_{\min}}^{X_{ij}} m_{ij}(x, \tau)dx = \sum_{(\alpha_{t}, \beta_{t}) \in \mathcal{A}_{ij}}\int_{\alpha_{t}}^{\beta_{t}}m_{ij}(x, \tau)dx + \sum_{[b_{t}, b_{t+1}] \in \mathcal{B}_{ij}}\int_{b_{t}}^{b_{t+1}}m_{ij}(x, \tau)dx
    \end{equation*}
    
    If we let $f'_{ij}(X_{ij})$ denote an arbitrary subgradient at $X_{ij}$, then we can write,
    \begin{equation*}
    f_{ij}(X_{\min}) - f_{ij}(X_{ij}) = -\int_{X_{\min}}^{X_{ij}}f_{ij}'(x)dx
    \end{equation*}
    since $f_{ij}$ is differentiable everywhere on $[X_{\min}, X_{ij}]$ except on at most a finite set.
    
    Thus,
    \begin{align*}
    &\phantom{{}={}}\max_{X \in \mathcal{D}} \abs*{\tilde{f}(X, \tau) - f(X)}\\
    &\leq \max_{X \in \mathcal{D}} \abs*{\sum_{i=1}^{m}\sum_{j=1}^{n}\int_{X_{\min}}^{X_{ij}} m_{ij}(x, \tau)dx + f_{ij}(X_{\min}) - f_{ij}(X_{ij})}\\
    &\leq \max_{X \in \mathcal{D}} \abs*{\sum_{i=1}^{m}\sum_{j=1}^{n}\int_{X_{\min}}^{X_{ij}} m_{ij}(x, \tau)dx - \int_{X_{\min}}^{X_{ij}}f_{ij}'(x)dx}\\
    &\leq \max_{X \in \mathcal{D}}\sum_{i=1}^{m}\sum_{j=1}^{n}\Bigg(\sum_{[\alpha_{t}, \beta_{t}] \in \mathcal{A}_{ij}}\int_{\alpha_{t}}^{\beta_{t}}\abs*{m_{ij}(x, \tau) - f'(x)}dx \\
    &\phantom{{}{\max}{}} + \sum_{(b_{t}, b_{t+1}) \in \mathcal{B}_{ij}}\int_{b_{t}}^{b_{t+1}}\abs*{m_{ij}(x, \tau) - f'(x)}dx \Bigg)\\  
    \end{align*}
    Let
    $\abs{\mathcal{A}_{ij}} $ and  $\abs{\mathcal{B}_{ij}}$  denote the number of subintervals for  ${\mathcal{A}_{ij}} $ and ${\mathcal{B}_{ij}}$ respectively. Then,    
    \begin{align*}
    \max_{X \in \mathcal{D}} \abs*{\tilde{f}(X, \tau) - f(X)}&\leq \max_{X \in \mathcal{D}}\sum_{i=1}^{m}\sum_{j=1}^{n}\left(\abs{\mathcal{A}_{ij}}\Delta_{f}\tau + \abs{\mathcal{B}_{ij}}D\Delta_{f}\tau \right)\\
    &\leq mn(M + 1)(1+D)\Delta_{f}\tau
    \end{align*}
    where in the second last line, we use Lemma \ref{lem:cvg_to_grad}, and the last line we have that $\abs{\mathcal{A}_{ij}}, \abs{\mathcal{B}_{ij}} \leq M + 1$ since $M$ is the maximum number of points of nondifferentiability for all $f_{ij}$.     
    \end{enumerate}
\end{proof}

\subsection*{Proof of Theorem \ref{thm:final_conv}}
For this proof, we utilize the standard Frank-Wolfe convergence theorem as seen in \cite{thesisMJ-supp}.
\begin{theorem}[\cite{thesisMJ-supp}]
\label{thm:fw_conv}
Let $x^{(k)}$ be the $k\textsuperscript{th}$ iterate given by the standard Frank-Wolfe algorithm, and let $x^{*} \in \argmin_{z \in \mathcal{D}} f(z)$. Then,
\begin{equation*}
    f(x^{(k)}) - f(x^{*}) \leq \frac{4C_{f}}{k + 2}.
\end{equation*}
\end{theorem}

\begin{proof}(Theorem \ref{thm:final_conv}).
Let $X^{*}_{k} \in \argmin_{X \in \mathcal{D}} \tilde{f}^{(k)}(X)$ and for any two functions $f$ and $g$ that are defined on $\mathcal{D}$, let $\norm{f - g}_{\infty} = \max_{X \in \mathcal{D}}\abs{f(X) - g(X)}$.

We have,
\begin{align*}
    f(X^{(k)}) - f(X^{*}) &\leq \tilde{f}^{(k)}(X^{(k)}) - f(X^{*}) \\
    &\leq \tilde{f}^{(k)}(X^{(k)}) - \tilde{f}^{(k)}(X^{*}) + \norm{\tilde{f}^{(k)} - f}_{\infty}\\
    &\leq \tilde{f}^{(k)}(X^{(k)}) - \tilde{f}^{(k)}(X_{k}^{*}) + \norm{\tilde{f}^{(k)} - f}_{\infty}
\end{align*}
Since $\tilde{f}^{(k)} = \tilde{f}^{(k')}$ for all $k \geq k'$, the bound becomes,
\begin{equation*}
    f(X^{(k)})  - f(X^{*}) \leq \tilde{f}^{(k')}(X^{(k)}) - \tilde{f}^{(k')}(X_{k}^{*}) + \norm{\tilde{f}^{(k')} - f}_{\infty}
\end{equation*}
From Theorem \ref{thm:unif_conv} implies that,
\begin{equation*}
  \norm{\tilde{f}^{(k)} - f}_{\infty} \leq  \frac{4mn(M+1)D^{2}\Delta_{f}}{k+2} \leq \frac{\epsilon}{2} \text{ when } k > \frac{8mn(M+1)D^{2}\Delta_{f}}{\epsilon} - 1
\end{equation*}
Thus, when $k \geq k'$,
\begin{equation}
\label{eqn:approx_error_k}
    \norm{\tilde{f}^{(k)} - f}_{\infty} = \norm{\tilde{f}^{(k')} - f}_{\infty} \leq \frac{\epsilon}{2}.
\end{equation}

From Theorem \ref{thm:fw_conv},  
\begin{equation*}
    \tilde{f}^{(k)}(X^{(k)}) - \tilde{f}^{(k)}(X_{k}^{*}) \leq \frac{4C_{\tilde{f}^{(k)}}}{k + 2} = \frac{4C_{\tilde{f}^{(k')}}}{k + 2}~ \forall k \geq k'.
\end{equation*}
Note that,
\begin{equation}
\label{eqn:fw_error_k}
    \frac{4C_{\tilde{f}^{(k')}}}{k + 2} < \frac{\epsilon}{2} \text{ when } k \geq \frac{8C_{\tilde{f}^{(k')}}}{\epsilon} - 1
\end{equation}
Hence, at most  $\frac{8C_{\tilde{f}^{(k')}}}{\epsilon} - 1$ additional iterations are required, and setting $k \geq k' + \frac{8C_{\tilde{f}^{(k')}}}{\epsilon} - 1$  guarantees that $\tilde{f}^{(k)}(X^{(k)}) - \tilde{f}^{(k)}(X_{(k)}^{*}) \leq \frac{\epsilon}{2}$. Combining \eqref{eqn:approx_error_k} and \eqref{eqn:fw_error_k}, we get the desired bound.
\end{proof}

\newpage
\subsection*{Derivation of $\tilde{f}$ for $f(X) = \norm{X}_{1}$}
We have $f(X) = \norm{X}_{1} = \sum_{ij} f_{ij}(X)$, with each $f_{ij}(X_{ij}) = \abs{X_{ij}}$. It it is straightforward to verify that $a_i = - \delta$ and,
\begin{equation*}
    m_{i}(X_{ij}, \tau) = \begin{cases}
    -1, \text{ when $X_{ij} < -\tau$}\\
    \frac{X_{ij}}{\tau}, \text{ when $-\tau \leq X_{ij} \leq \tau$}\\
    1, \text{ when $X_{ij} > \tau$}.
    \end{cases}
\end{equation*}
The lower bound of each $X_{ij}$ is $-\delta$, so we can compute the integrals using the following cases. First consider when $X_{ij} < -\tau$. Then, 
\begin{equation*}
    \int_{-\delta}^{X_{ij}}m_{i}(t, \tau)dt + f_{ij}(-\delta) = -X_{ij} - \delta + \delta = -X_{ij}.
\end{equation*}
When $-\tau \leq X_{ij} \leq \tau$, we have,
\begin{align*}
    &\phantom{{}={}}\int_{-\delta}^{X_{ij}}m_{i}(t, \tau)dt + f_{ij}(-\delta) \\
    &= \int_{-\delta}^{-\tau}m_{i}(t, \tau)dt + \int_{-\tau}^{X_{ij}}m_{i}(t, \tau)dt + f_{ij}(-\delta)\\
    &= \frac{X_{ij}^{2}}{2\tau} + \frac{\tau}{2}.
\end{align*}
Finally, when $X_{ij} > \tau$
\begin{align*}
    &\phantom{{}={}}\int_{-\delta}^{X_{ij}}m_{i}(t, \tau)dt + f_{ij}(-\delta)\\
    &= \int_{-\delta}^{\tau}m_{i}(t, \tau)dt + \int_{\tau}^{X_{ij}}m_{i}(t, \tau)dt + f_{ij}(-\delta)\\
    &= X_{ij}.
\end{align*}

Thus, the expression for $\hat{f}(x,\tau)$ is,
\begin{equation*}
    \hat{f}(x, \tau) = \sum_{ij} \left(\frac{X_{ij}^{2}}{2\tau}  + \frac{\tau}{2} \right)\cdot\mathbbm{1}_{\abs{X_{ij}} <  \tau} + \abs{X_{ij}}\cdot\mathbbm{1}_{\abs{X_{ij}} \geq \tau}
\end{equation*}
where $\mathbbm{1}$ is the indicator function.

\newpage
\renewcommand{\arraystretch}{1.3}
\subsection*{Parameters used for Experiments}
\begin{table}[ht]
    \centering
    \begin{tabular}{c|c|l|c|r}
    \specialrule{.2em}{.1em}{.1em}
         Parameter & Methods & Description & $n$ & Value  \\
         \specialrule{.2em}{.1em}{.1em}
         $\lambda_{1}$  & All methods & Regularization Parameter for $\ell_{1}$ norm & all &  0.4\\
         \hline
         $\lambda_{2}$  & GenFB & Regularization Parameter for trace norm & all  &  10\\
         \hline
          \multirow{6}{*}{$\delta$} & \multirow{6}{*}{Frank-Wolfe Variants} &  \multirow{6}{*}{Trace norm constraint(avg. observed)} & 750 & 43.92\\\cline{4-5}           
         & & & 1000 & 73.17\\\cline{4-5}
         & & & 1250 & 100.43\\\cline{4-5}
         & & & 1500 & 134.06\\\cline{4-5}
         & & & 1750 & 169.40\\\cline{4-5}
         & & & 2000 & 201.46\\\cline{4-5}
         \hline
         $\mu$ & SCCG & Smoothing parameter & all &  0.01\\
         \specialrule{.2em}{.1em}{.1em}
    \end{tabular}
    \caption{Parameters used for sparse covariance estimation}
    \label{tab:cov_params}
\end{table}

\begin{table}[ht]
    \centering
    \begin{tabular}{c|c|c|l|r}
    \specialrule{.2em}{.1em}{.1em}
         Parameter & Methods & $\sigma$ & Description & Value  \\
         \specialrule{.2em}{.1em}{.1em}
         \multirow{3}{*}{$\lambda_{1}$} & \multirow{3}{*}{All methods} & 0 & \multirow{3}{*}{Regularization parameter for $\ell_{1}$} & 0.01\\
         & & 0.05 & & 0.05\\
         & & 0.01 & & 0.1\\
         \hline
         \multirow{3}{*}{$\lambda_{2}$} & \multirow{3}{*}{GenFB} & 0 & \multirow{3}{*}{Regularization parameter for trace norm} & 5\\
         & & 0.05 & & 25\\
         & & 0.01 & & 50\\
         \hline
          \multirow{3}{*}{$\delta$} & \multirow{3}{*}{Frank-Wolfe Variants} & 0 & \multirow{3}{*}{Trace norm constraint (average observed)} & 1052.88\\         
         & & 0.05 & & 590.13\\
         & & 0.01 & & 639.87\\
         \hline
         \multirow{2}{*}{$\mu$} & \multirow{2}{*}{SCCG} & 0 and 0.1 & \multirow{2}{*}{Smoothing parameter} & 0.001\\
          &  & 0.05 &  & 0.01\\
         \specialrule{.2em}{.1em}{.1em}
    \end{tabular}
    \caption{Parameters used for graph link prediction}
    \label{tab:graph_params}
\end{table}

\begin{table}[!ht]
    \centering
    \begin{tabular}{c|c|c|l|r}
    \specialrule{.2em}{.1em}{.1em}
         Dataset & Parameter &   Description & Methods & Value  \\
         \specialrule{.2em}{.1em}{.1em}
         \multirow{4}{*}{MovieLens 100k} & $\lambda$ & Parameter for $\norm{P_{\Omega^{c}}(X)}_{2}^{2}$ &  RLRMC &  0.001\\
         \cline{2-5}
          & \multirow{2}{*}{$r$} &  \multirow{2}{*}{Fixed rank} & RLRMC & 2\\
          \cline{4-5}
          & & & GLRL & 10\\
         \cline{2-5}         
          & $\delta$ &  Trace norm constraint & FWUA & 1,200\\
         \specialrule{.1em}{.1em}{.1em}
         \multirow{4}{*}{MovieLens 1M} & $\lambda$ & Parameter for $\norm{P_{\Omega^{c}}(X)}_{2}^{2}$ & RLRMC & 0.001\\
         \cline{2-5}
          & \multirow{2}{*}{$r$} &  \multirow{2}{*}{Fixed rank} & RLRMC & 5\\
          \cline{4-5}
          & & & GLRL & 25\\
         \cline{2-5}         
          & $\delta$ &  Trace norm constraint & FWUA & 6,400\\
         \specialrule{.1em}{.1em}{.1em}
         \multirow{4}{*}{MovieLens 10M} & $\lambda$ & Parameter for $\norm{P_{\Omega^{c}}(X)}_{2}^{2}$ & RLRMC & 0.005\\
         \cline{2-5}
          & \multirow{2}{*}{$r$} &  \multirow{2}{*}{Fixed rank} & RLRMC & 10\\
          \cline{4-5}
          & & & GLRL & 50\\
         \cline{2-5}         
          & $\delta$ &  Trace norm constraint & FWUA & 15,000\\
         \specialrule{.2em}{.1em}{.1em}
    \end{tabular}
    \caption{Parameters used for matrix completion}
    \label{tab:mc_params}
\end{table}

\newpage
\subsection*{Full Experimental Results for Covariance Estimation}
\begin{figure*}[!h]
\centering
    \begin{subfigure}[b]{0.3\textwidth}
        \centering
        \includegraphics[width=\textwidth]{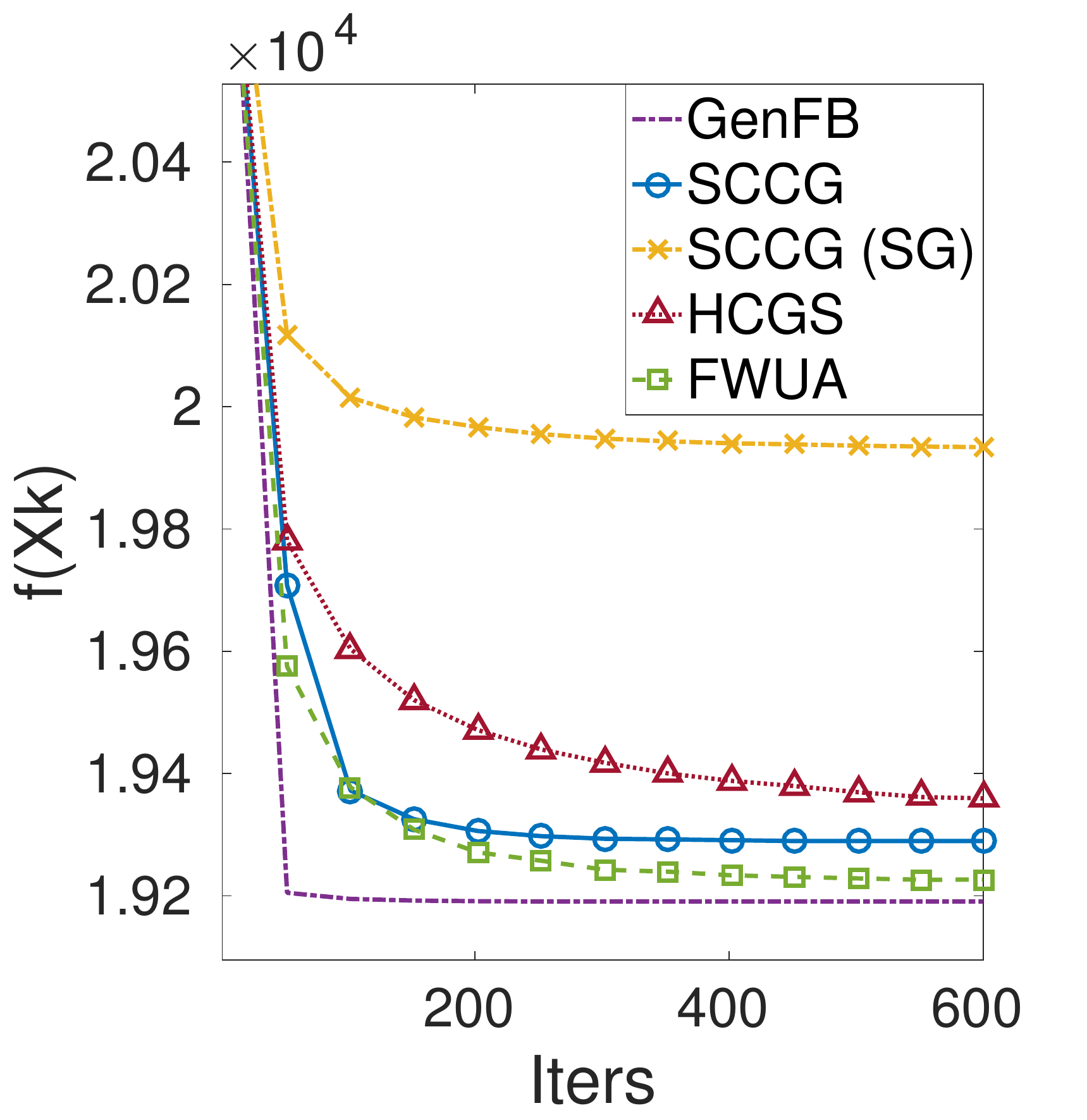}        
        \caption{$n=750$}
    \end{subfigure}%
    ~ 
        \begin{subfigure}[b]{0.3\textwidth}
        \centering
        \includegraphics[width=\textwidth]{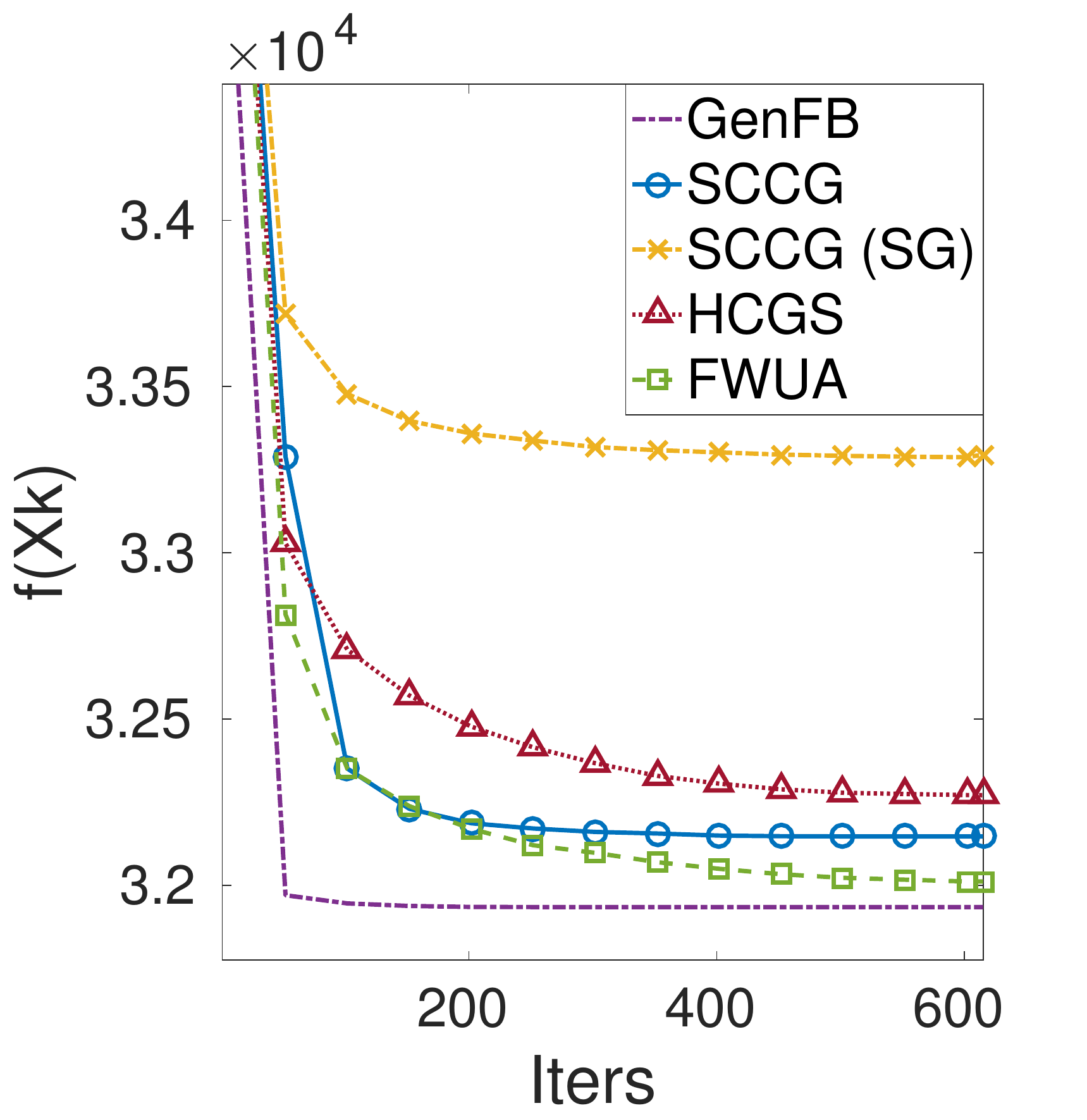}        
        \caption{$n=1000$}
    \end{subfigure}%
    ~
    \begin{subfigure}[b]{0.3\textwidth}
        \centering
        \includegraphics[width=\textwidth]{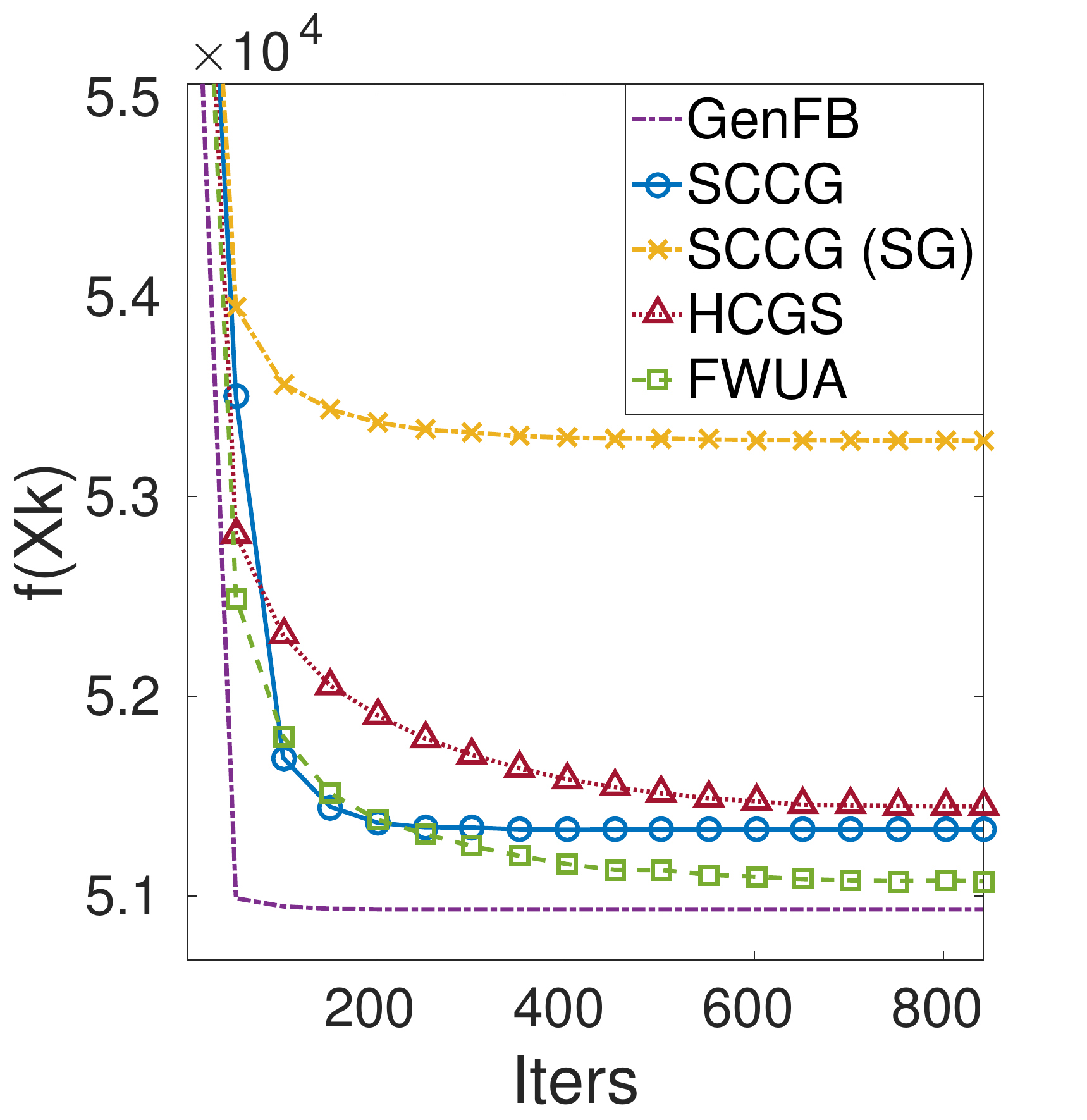}    
        \caption{$n=1250$}
    \end{subfigure}%

    \begin{subfigure}[b]{0.3\textwidth}
        \centering
        \includegraphics[width=\textwidth]{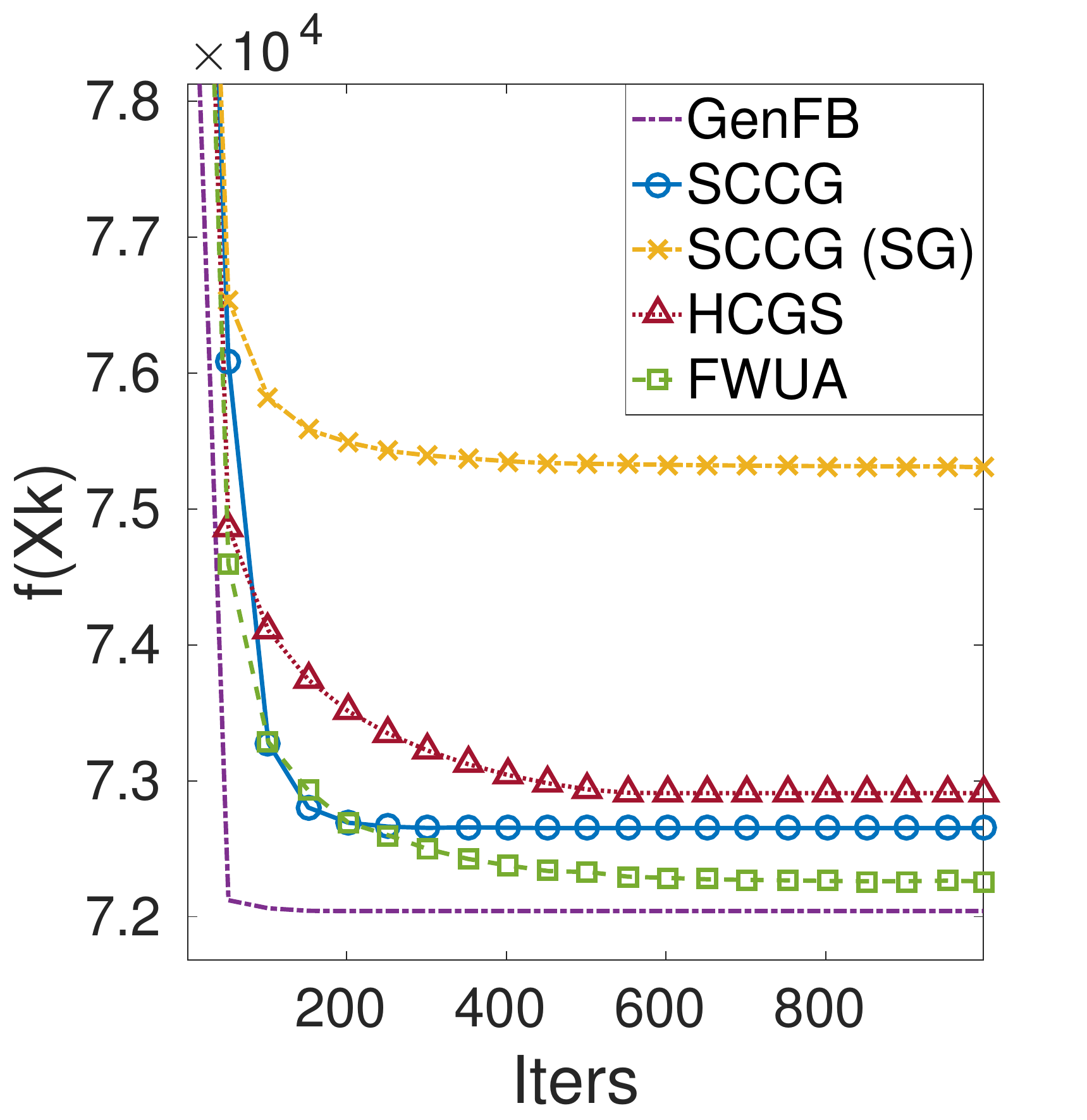}        
        \caption{$n=1500$}
    \end{subfigure}%
    \begin{subfigure}[b]{0.3\textwidth}
        \centering
        \includegraphics[width=\textwidth]{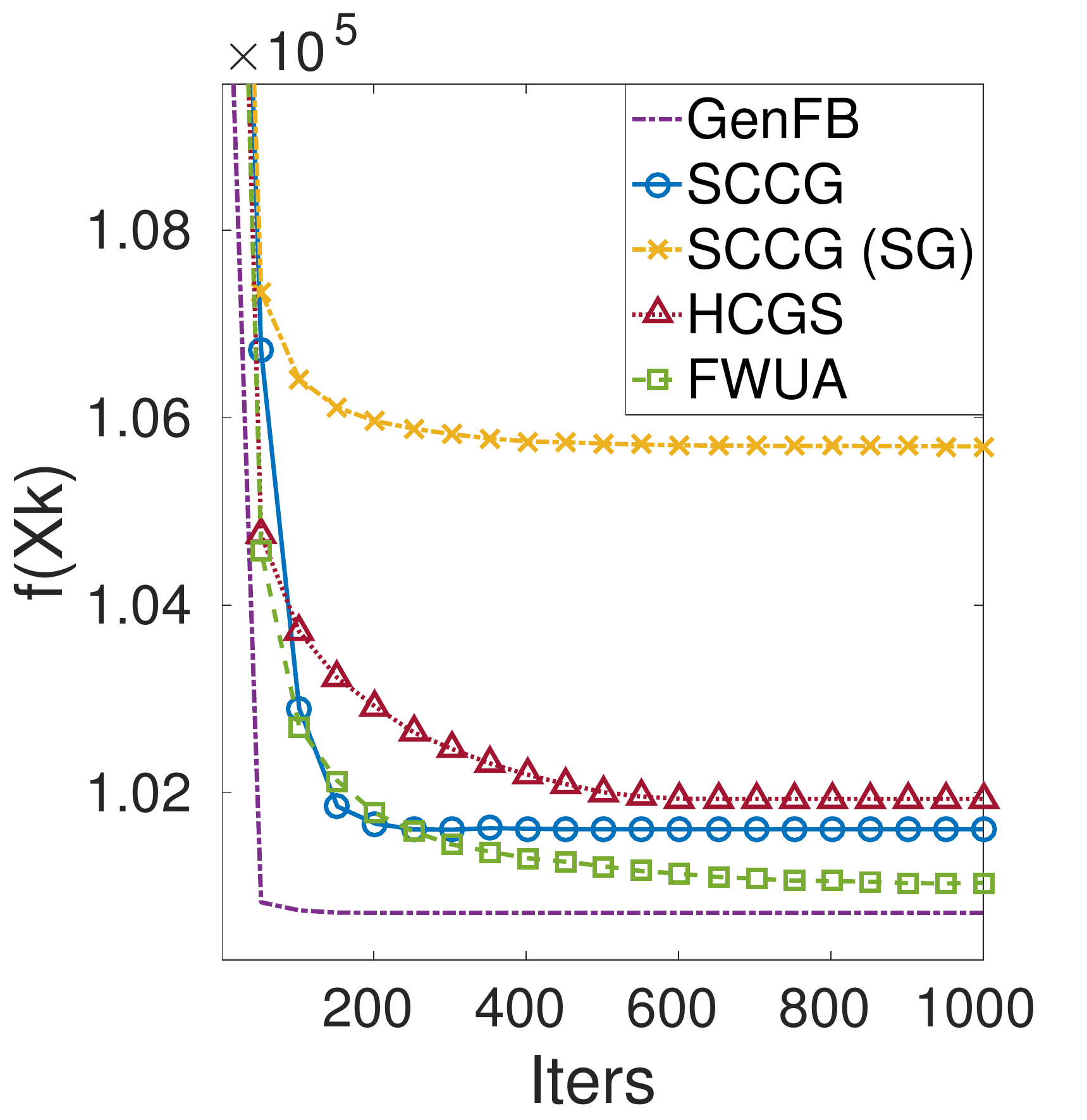}        
        \caption{$n=1750$}
    \end{subfigure}%
    ~
    \begin{subfigure}[b]{0.3\textwidth}
        \centering
        \includegraphics[width=\textwidth]{cov-n=2000.eps}        
        \caption{$n=2000$}
    \end{subfigure}%
    \caption{Convergence of various synthetic dataset sizes.}
\end{figure*}

\end{document}